\documentclass[hidelinks,onefignum,onetabnum]{siamart250211}

\usepackage{graphicx} 
\graphicspath{{figures/}{./}} 

\usepackage{lipsum}
\usepackage{amsfonts}
\usepackage{graphicx}
\usepackage{epstopdf}
\usepackage{algorithmic}
\usepackage{physics}
\usepackage{listings}
\lstdefinestyle{rampdepython}{
  language=Python,
  basicstyle=\ttfamily\small,
  keywordstyle=\bfseries,
  commentstyle=\itshape,
  showstringspaces=false,
  columns=fullflexible,
  keepspaces=true,
  frame=single,
  xleftmargin=0pt,
  xrightmargin=0pt,
  aboveskip=0pt,
  belowskip=0pt,
}
\ifpdf
  \DeclareGraphicsExtensions{.eps,.pdf,.png,.jpg}
\else
  \DeclareGraphicsExtensions{.eps}
\fi

\newsiamremark{remark}{Remark}
\newsiamremark{hypothesis}{Hypothesis}
\crefname{hypothesis}{Hypothesis}{Hypotheses}
\newsiamthm{claim}{Claim}
\newsiamremark{fact}{Fact}
\crefname{fact}{Fact}{Facts}

\headers{Mixed Precision Training of Neural ODEs}{E. Celledoni Et Al.}

\title{Mixed Precision Training of Neural ODEs\thanks{Submitted to the editors November 12, 2025.
\funding{LR's and NTY's work was supported by the Office of Naval Research award N00014-24-1-2221/ P00003 and also in part by NSF award DMS 2038118. The project was supported by the Horizon Europe, MSCA-SE project 101131557 (REMODEL).}}}

\author{Elena Celledoni\thanks{NTNU – Norwegian University of Science and Technology, Trondheim, Norway (\email{elena.celledoni@ntnu.no}, \url{https://www.ntnu.edu/employees/elena.celledoni})} \and Brynjulf Owren\footnotemark[1] \thanks{(\email{brynjulf.owren@ntnu.no}, \url{https://www.ntnu.edu/employees/brynjulf.owren})} \and Lars Ruthotto \thanks{Emory University, Atlanta, GA, USA (\email{lruthotto@emory.edu}, \url{https://www.math.emory.edu/\string~lruthot/})} \and Nicole Tianjiao Yang\thanks{University of Tennessee, Knoxville, TN, USA (\email{nicole.yang@utk.edu}, \url{https://nicoletyang.github.io})}
}

\usepackage{amsopn}
\DeclareMathOperator{\diag}{diag}

\usepackage{xcolor}
\definecolor{exactcolor}{RGB}{0,0,0}      
\definecolor{roundedcolor}{RGB}{0,0,0}    

\newcommand{\Phiexact}{\textcolor{exactcolor}{\Phi}}              
\newcommand{\Phirounded}{\textcolor{roundedcolor}{\tilde{\Phi}}}    

\newcommand{\Cexact}{\textcolor{exactcolor}{C}}                   

\newcommand{\DPhiexact}{\textcolor{exactcolor}{D\Phi}}           
\newcommand{\DPhirounded}{\textcolor{roundedcolor}{D\tilde{\Phi}}} 
\newcommand{\nablaCexact}{\textcolor{exactcolor}{\nabla C}}      
\newcommand{\nablaCreounded}{\textcolor{roundedcolor}{\nabla \tilde{C}}}  

\newcommand{\bfyexact}[1]{\textcolor{exactcolor}{\bfy_{#1}}}     
\newcommand{\bfyrounded}[1]{\textcolor{roundedcolor}{\tilde{\bfy}_{#1}}}  

\newcommand{\bfaexact}[1]{\textcolor{exactcolor}{\bfa_{#1}}}     
\newcommand{\bfarounded}[1]{\textcolor{roundedcolor}{\tilde{\bfa}_{#1}}}  

\newcommand{\bfxrounded}[1]{\textcolor{roundedcolor}{\tilde{\bfx}_{#1}}}

\usepackage{subcaption}
\usepackage{siunitx} 
\usepackage{url}

\usepackage{hyperref}
\ifpdf
\hypersetup{
  pdftitle={Mixed Precision Training of Neural ODEs},
  pdfauthor={E. Celledoni, B. Owren, L. Ruthotto, T. Yang}
}
\fi

\usepackage{algorithmic}
\usepackage{booktabs}
\usepackage{multirow}
\usepackage{MnSymbol}

\newcommand{\R}{\mathbb{R}}
\newcommand{\E}{\mathbb{E}}
\newcommand{\bfx}{\mathbf{x}}
\newcommand{\bfy}{\mathbf{y}}
\newcommand{\bfW}{\mathbf{W}}
\newcommand{\bfI}{\mathbf{I}}
\newcommand{\bfJ}{\mathbf{J}}
\newcommand{\bfb}{\mathbf{b}}
\newcommand{\bft}{\mathbf{t}}
\newcommand{\bfa}{\mathbf{a}}
\newcommand{\bfe}{\mathbf{e}}
\newcommand{\bfg}{\mathbf{g}}
\newcommand{\bfz}{\mathbf{z}}
\newcommand{\bfr}{\mathbf{r}}
\newcommand{\bfU}{\mathbf{U}}
\newcommand{\bfone}{\mathbf{1}}
\newcommand{\bfdelta}{\boldsymbol{\delta}}
\newcommand{\bfK}{\mathbf{K}}
\newcommand{\bfM}{\mathbf{M}}
\newcommand{\bfP}{\mathbf{P}}

\newcommand{\bftheta}{\boldsymbol{\theta}}

\newcommand{\uhigh}{u_{\rm H}}
\newcommand{\ulow}{u_{\rm L}}

\newtheorem{example}[theorem]{Example}
\newtheorem{assumption}[theorem]{Assumption}
\begin{document}

\maketitle

\begin{abstract}
Exploiting low-precision computations has become a standard strategy in deep learning to address the growing computational costs imposed by ever larger models and datasets.
However, na{\"i}vely performing all computations in low precision can lead to roundoff errors and instabilities.
Therefore, mixed precision training schemes usually store the weights in high precision and use low-precision computations only for whitelisted operations.
Despite their success,  these principles are currently not reliable for training continuous-time architectures such as neural ordinary differential equations (Neural ODEs).
This paper presents a mixed precision training framework for neural ODEs consisting of explicit ODE solvers and a custom backpropagation scheme and shows their effectiveness in a range of learning tasks.
Our scheme uses low-precision computations for evaluating the velocity, parameterized by the neural network, and for storing intermediate states, while numerical reliability is provided by custom dynamic adjoint scaling and by accumulating the solution and gradients in higher precision.
These contributions address two key challenges in training neural ODEs: the computational cost of repeated network evaluations and the growth of memory requirements with the number of time steps or layers.
Along with the paper we publish our extendable, open-source PyTorch package \texttt{rampde}, whose syntax resembles that of leading packages to provide a drop-in replacement in existing codes.
We demonstrate the reliability and effectiveness of our scheme using challenging test cases and on neural ODE applications in image classification and generative models, achieving approximately 50\% memory reduction and up to 2× speedup while maintaining accuracy comparable to single-precision training.
\end{abstract}



\section{Introduction}
Mixed precision training (MPT) has become a standard practice in deep learning to mitigate the growth of computational costs associated with increasing models and data sizes. It is increasingly used in various learning tasks, including natural language processing, computer vision~\cite{JiaEtAl2018}, and it has also made inroads into scientific machine learning~\cite{hayford_speeding_2024}.

Mixed precision training aims to reduce the computational costs of training by evaluating as many compute-intensive parts of the deep network in low-precision arithmetic as possible, while using as few high-precision operators as necessary to preserve training stability and match the accuracy of networks trained using high-precision arithmetic. The two most significant aspects in which MPT lowers computational costs are enabling the training of larger models and larger batch sizes by storing hidden activations and features in low precision and reducing computational time when models are sized to take advantage of tensor compute units available in modern hardware accelerators such as graphical processing units (GPUs).

In addition to the growing size of models and data, the widespread use of MPT in deep learning can be explained by several characteristics of the problem~\cite{higham_mixed_2022}. First, training is inherently stochastic due to the presence of noise in the data and updates in stochastic gradient methods. Second, deep networks are composed of a small subset of linear algebra routines, often involving dense arrays, which are ideal targets for developing specialized hardware that reduces computation times and energy consumption. Third, the optimization problem is only a surrogate for the actual goal of learning and therefore does not need to be solved with high accuracy.

Most MPT algorithms today maintain the network weights in high precision to prevent cancellation errors during optimization steps, but quantize the weights to low precision to accelerate the forward and backward passes. During network evaluations,  whitelisted, compute-intensive operators are evaluated in low precision,  high precision is used for delicate operations such as accumulation and some nonlinearities, and the loss function is scaled before backpropagation to maximize the range and avoid overflow and underflow~\cite{MicikeviciusEtAl2018}. This strategy balances computational efficiency with numerical stability: compute-intensive operations benefit from hardware acceleration while sensitive accumulations maintain accuracy through higher precision. Developing MPT algorithms is also simplified by the integration of these algorithms into modern deep learning software, such as PyTorch, JAX, and Keras.

Despite these advances, existing MPT techniques can be unreliable for training neural ordinary differential equations (Neural ODEs), which are continuous-time models in which the forward propagation involves the numerical solution of initial value problems whose dynamics are parameterized by a neural network~\cite{e_proposal_2017, haber_stable_2018, ChenEtAl2019}. Due to the nonlinearity of neural networks, Neural ODEs typically rely on explicit time integrators, which can necessitate many steps and thereby exacerbate computational costs and memory requirements.  Calculating the gradient of the loss function with respect to the network weights requires solving a continuous or discrete adjoint equation. We demonstrate theoretically and experimentally that the straightforward application of MPT can lead to the accumulation of roundoff errors during time integration, resulting in growing errors in the solution and gradients as the number of time steps increases.
The main contributions of this paper are:
\begin{enumerate}
\item \textbf{Mixed precision ODE solvers and backpropagation}: We design and analyze mixed precision explicit solvers for neural ODEs with a custom backpropagation scheme that stores weights, states, and adjoints in high precision while evaluating networks in low precision, achieving approximately 50\% memory reduction and up to 2× speedup in our largest example.

\item \textbf{Dynamic adjoint scaling}: We develop an adaptive scaling heuristic to maximally exploit the limited range of \texttt{float16} during backpropagation and prevent over- and underflow without requiring additional hyperparameter tuning.

\item \textbf{Theoretical analysis}: We show that roundoff errors remain in the order of the unit roundoff of the low precision and do not grow uncontrollably with the number of time steps.

\item \textbf{Implementation}: We provide the open-source PyTorch package \texttt{rampde} (robust automatic mixed precision for differential equations) that serves as a drop-in replacement for existing neural ODE implementations at
\begin{center}
    \url{https://github.com/EmoryMLIP/rampde}
\end{center}
\end{enumerate}

Our paper is structured as follows. In~\cref{sec:background}, we provide some background on mixed precision algorithms in scientific computing and deep learning. In~\cref{sec:mpt_node}, we describe our mixed precision training framework for neural ODEs. In~\cref{sec:roundoff} we perform a roundoff error analysis of our forward and backward schemes to show that errors are dominated by roundoff errors in the low-precision network evaluation and nearly independent of the number of integration steps. In~\cref{sec:implementation}, we describe our open-source implementation \texttt{rampde}. In~\cref{sec:experiments},  we demonstrate reliability and broad applicability to neural ODE training problems arising in image classification and generative modeling. In~\cref{sec:discussion}, we discuss our findings and outline the potential for future improvements.

\section{Background}\label{sec:background}

Let us introduce our main notation, briefly review the mixed precision training (MPT) framework for deep learning, and provide background for our contribution by discussing recent works on MPT in scientific machine learning.

\paragraph{Notation}
Without loss of generality, we describe unsupervised training of a neural network $F : \R^n \times \R^p \to \R^m$ with $n$ input features, $m$ output features, and $p$ parameters using samples $\bfx$ from the data distribution $\pi_{\rm data}$ by minimizing the regularized loss function
\begin{equation*}
    \mathcal{L}(\bftheta) = \E_{\bfx \sim \pi_{\rm data}}\left[\ell(F(\bfx,\bftheta))\right] + \frac{\alpha}{2} \|\bftheta\|^2.
\end{equation*}
Here, the choice of loss function $\ell : \R^m \to \R$ depends on the task, and we use a Tikhonov regularizer (in this context known as weight decay) controlled by the parameter $\alpha\geq 0$. The design of the neural network also varies depending on the context.
When training the network using stochastic approximation schemes such as stochastic gradient descent (SGD), the objective function and its gradient are approximated using a batch of independent and identically distributed (i.i.d.) samples $\mathcal{B}=\{\bfx_1, \ldots, \bfx_b\}$
\begin{equation*}
        \mathcal{L}_{\mathcal{B}}(\bftheta) = \frac{1}{|\mathcal{B}|} \sum_{x_i\in\mathcal{B}} \ell(F(\bfx_i,\bftheta)) + \frac{\alpha}{2} \|\bftheta\|^2.
\end{equation*}

\paragraph{Mixed Precision Training (MPT)}
Let us briefly review the principles of MPT proposed in~\cite{MicikeviciusEtAl2018} that are widely adopted in deep learning, discuss their implementation in modern machine learning software, and their use in scientific applications. These
principles establish the foundation for our neural ODE-specific adaptations,
particularly in handling the unique challenges of time integration.
The goal of MPT is to combine low- and high-precision arithmetic to reduce the costs (in terms of memory, energy, and computational time) of network training while achieving similar precision compared to pure high-precision training.

Following standard practice of deep learning today, we use 16-bit floating point formats (either \texttt{float16} or \texttt{bfloat16}) as low precision and \texttt{float32} as high precision.
The main difference between the low-precision formats is the allocation of bits between the mantissa and exponent.
The \texttt{float16} or IEEE-754 half precision has one sign, five exponent bits, and ten mantissa bits, while \texttt{bfloat16}, developed by Google~\cite{KalamkarEtAl2019}, has one sign, eight exponent bits, and seven mantissa bits.
The \texttt{bfloat16} format is appealing because its range coincides with that of \texttt{float32}, which usually alleviates the need to modify the network architecture and adjust hyperparameters and also simplifies conversion. In comparison, \texttt{float16}, uses more bits for the mantissa, resulting in higher precision but a smaller range.

We denote by $Q_{16} : \mathbb{R}^n \to \mathcal{F}_{16}^n$ and $Q_{32} : \mathbb{R}^n \to \mathcal{F}_{32}^n$ the quantization operators that map the elements of an input vector or tensor to the set of representable numbers in the 16-bit and 32-bit floating point representations, respectively.

For ease of presentation, consider a simple stochastic gradient descent (SGD) scheme for training the network, and note that MPT can be similarly used for other stochastic approximation methods. Given a master copy of the weights $\bftheta_k$ in high precision, the $k$th iteration of SGD in MPT performs the update
\begin{equation}
\label{eq:sgd_mpt}
\bftheta_{k+1} = \bftheta_k - \gamma_k Q_{32}(  \nabla (S_k \cdot \mathcal{L}_{\mathcal{B}_k}(Q_{16}(\bftheta_k))))/S_k,
\end{equation}
where $\mathcal{B}_k$ is a randomly sampled batch from the training set, $\gamma_k>0$ is the step size (or learning rate), and $S_k$ is a loss scaling factor that aims to maximize the range of the low-precision floating point system and avoid underflow and overflow during derivative computations. In MPT, the input features are also converted to match the precision of the weights. The sequence of learning rates is chosen by the user and the loss scale $S_k$ is updated heuristically during training. Choosing $S_k$ effectively is crucial since too small values can lead to underflow, which means that small entries in the gradient are lost, and too large values can lead to overflow, which means that updates need to be discarded. The most widely used heuristics therefore halve $S_k$ when gradients overflow and double it when no overflow occurred in a user-defined number of the last few consecutive steps.

Since a single scaling factor cannot keep the scales of all hidden layers in a meaningful range for deep and complex network architectures, adaptive loss scaling schemes use layer-dependent scaling factors~\cite{ZhaoEtAl2019}. The proposed heuristic for computing the scaling factors automatically during training reduced the time to convergence and improved accuracy compared to standard MPT techniques in examples from image classification and object detection.

Modern deep learning software packages simplify the use and experimentation with mixed precision algorithms by providing automatic mixed precision (AMP) contexts.
These AMP contexts whitelist compute-intensive operations that are usually safe in low precision (e.g., dense matrix-matrix operations, convolutions, etc.) and blacklist more delicate operations (e.g., norms, accumulations, nonlinearities such as softmax, etc.)
When evaluating the neural network and its gradients in an AMP context, the inputs of the whitelisted operations are automatically converted and performed with that precision.
Although this adds convenience, it can complicate the interpretation of mixed precision schemes such as~\eqref{eq:sgd_mpt} since only parts of weights and operations may be performed in low precision.

Another practical aspect of MPT on modern GPU hardware is that some linear algebra operations leverage mixed precision algorithms internally.
One prominent example is the introduction of block fused multiply-adds (FMAs) in tensor compute units, which is described and analyzed in detail in~\cite{blanchard_mixed_2020}. Block FMAs add accuracy for operations on 16-bit tensors, as accumulation is done in single precision intermediates that are then rounded to half precision~\cite{KalamkarEtAl2019}. With appropriately sized matrix dimensions, observed speed-ups of tensor compute units can be up to 8 times with a reduced loss of accuracy compared to pure 16-bit operations~\cite{blanchard_mixed_2020, MicikeviciusEtAl2018}.

Mixed precision training has also been used in scientific machine learning (SciML) tasks such as solving partial differential equations (PDEs) with Physics-Informed Neural Networks (PINNs) and Deep Operator Networks (DeepONet).
As expected, pure low-precision training can fail in SciML tasks, but the use of the MPT framework can lead to effective training at reduced costs when hyperparameters and loss functions are chosen adequately~\cite{hayford_speeding_2024}.

\paragraph{Neural ODEs}
Instead of defining the neural network as the composition of a finite number of layers, Neural ODEs are continuous-time models that define $F(\bfx,\bftheta) = \bfy(T)$ as the terminal state of the initial value problem
\begin{equation}
    \label{eq:NODE}
    \frac{d}{dt} \bfy(t) = f(t, \bfy(t), \bftheta(t)),\quad \text{ for } t \in (0,T], \quad \text{ and } \quad \bfy(0) = \bfx.
\end{equation}
Here, $T$ is the terminal time, $f: [0,T] \times \R^n \times \R^p \to \R^n$ is a neural network with potentially time-dependent weights $\bftheta : [0,T] \to \R^p$.
Although first instances of continuous-time recurrent neural networks can be dated back at least to the 1980s (see historical notes in~\cite{MassaroliEtAl2021}) and ODE-based architectures have also been used for time series regression in the 1990s (see, for example,~\cite{RICOMARTINEZ1992}), their recent interpretation as continuous residual networks has contributed to their reemergence; see~\cite{e_proposal_2017, haber_stable_2018}.
The term neural ODE was coined in~\cite{ChenEtAl2019}, and their use has spread to other deep learning tasks such as generative modeling~\cite{GrathwohlEtAl2018, OnkenEtAl2020OTFlow, Finlay:2020wt}, mean field games~\cite{RuthottoEtAl2020MFG}, and optimal control~\cite{OnkenEtAl2020ieee}.
A widely used implementation in the PyTorch framework is the \texttt{torchdiffeq} project~\cite{torchdiffeq}.
Generalizations to rough paths and stochastic dynamics also provide the basis of Neural SDEs~\cite{Kidger2022} and score-based generative models~\cite{SongEtAl2020}.

A key distinguishing factor in neural ODE implementations is the order of differentiation and discretization.
Optimize-then-discretize algorithms such as~\cite{ChenEtAl2019, GrathwohlEtAl2018}, derive the gradient of the objective function in continuous time, and use adaptive time integrators for the forward and adjoint process to approximate the update in each step.
Although this simplifies their use, the implicit function theorem requires that both ODEs be solved relatively accurately to approximate the continuous gradient with sufficient accuracy~\cite{Gholami:2019wm,Onken2020DO}.
Discretize-then-optimize approaches first discretize the problem in time and then compute the gradient of the discrete problem analytically or through automatic differentiation.
The main advantage is that the accuracy of the time integrator can be chosen arbitrarily without compromising the accuracy of the (discrete) gradient. In general, and in particular in machine learning tasks, using a lower accuracy for the ODE solver is attractive and can lead to significant speedups when the data is noisy and the ODE is not derived from first principles.
In an MPT framework it is not clear that the ODE can be solved as accurately as required in the optimize-then-discretize approach (i.e., the adaptive solver may fail because required step sizes are rounded to zero or accuracy is limited by low-precision evaluations of $f$) and therefore, we limit the discussion to discretize-then-optimize methods.

Mixed precision implementations of additive Runge-Kutta (RK) schemes have also been proposed in~\cite{ grant_perturbed_2022} and their performance has been analyzed in~\cite{burnett2021performance-755}.
These works aim to speed up implicit steps of the solvers using low precision (here single precision), while using high precision (here double or quadruple) for explicit steps to maintain overall accuracy. Performance evaluations on IBM Power9 and Intel chips show speed-ups and reductions in energy costs for specific solver parameters and software kernels.
Such schemes could be attractive for neural ODEs in an optimize-then-discretize approach; however, implicit steps are rarely used in deep learning due to the nonlinearity introduced by the neural network, which often has limited exploitable structure for the nonlinear solver. As noted above, following the ODE trajectory with high accuracy is usually less important in a learning context than in physical models, as long as the neural ODE leads to effective learning.
A complementary line of work~\cite{croci2022mixed} targets explicit stabilized Runge--Kutta (RKC) methods for stiff problems, which avoid implicit solves altogether by using many cheap explicit stages whose stability region grows as $s^2$ with the stage count $s$. Since the interior stability stages exist only to enlarge the stability region and do not contribute to accuracy, they are demoted to low precision via a linearised Jacobian increment, yielding a mixed-precision scheme with no low-precision implicit solve.
In order to apply any of these schemes for neural ODE training, the principles we propose in this work for simpler time integrators can be applied to derive an effective mixed-precision backward propagation.

\section{Mixed Precision Training of Neural ODEs}\label{sec:mpt_node}

We describe our mixed precision scheme for optimization problems governed by a neural ODE such as~\eqref{eq:NODE}, the custom backpropagation scheme to compute gradients of the loss with respect to time, input features, and weights, and our dynamic adjoint scaling to maximize the  range of the low precision.

Let $\bft\in\R^{N+1}$ with $0 = \bft_0 < \bft_1 < \cdots < \bft_N = T$ be a partition of the time interval and let $h_i = \bft_{i+1}-\bft_i$ be the size of the $i$th step where $i=0,\ldots, N-1$.
For ease of presentation, we assume a batch-size of one example.
Given the initial state $\bfy_0 = \bfy(0) =\bfx$,  we are interested in approximating the ODE solution $\bfy_i \approx \bfy(\bft_i)$ with an explicit single-step method of the form
\begin{equation*}
    \bfy_{i+1} =  \bfy_i + h_i \Phi(f, \bfy_i, \bft_i, h_i, \bftheta),
\end{equation*}
where different choices of the increment function, $\Phi$,
lead to  different time stepping schemes. For example,  $\Phi(f, \bfy_i, \bft_i, h_i, \bftheta) = f(\bft_i,\bfy_i,\bftheta(\bft_i))$ leads to the forward Euler method.

To limit roundoff errors during time integration, we store the current state and compute the accumulation in high-precision, which gives the mixed precision step
\begin{equation*}
    \bfy_{i+1} = \bfy_i + h_i Q_{32} (\Phi(f, Q_{16}(\bfy_i), \bft_i, h_i, Q_{16}(\bftheta))).
\end{equation*}
Similar to~\eqref{eq:sgd_mpt}, the computationally expensive part of evaluating the neural network is performed in low precision.
To save memory, our method stores and returns the intermediate states $\{\bfy_i(\bft,\bfx,\bftheta)\}_{i=0}^N$ in low precision; see~\cref{alg:forward-pass} for a summary.

\begin{algorithm}[t]
  \caption{Mixed Precision Forward Pass}
  \label{alg:forward-pass}
  \begin{algorithmic}[1]
    \REQUIRE time steps $0 \leq \bft_0 < \bft_1 < \dots < \bft_N = T$, initial state $\bfy_0 = \bfx$, weights $\bftheta$, increment function $\Phi(f,\bfy,t,h,\bftheta)$
    \ENSURE state trajectory $\{\,Q_{16}(\bfy_i)\}_{i=0}^N$ in low precision
    \STATE $\bfy \gets Q_{32}(\bfx)$
      \COMMENT{loop invariant: $\bfy$ is the high-precision representation of $\bfy_i$}
    \FOR{$i = 0$ \TO $N-1$}
      \STATE $h_i \gets \bft_{i+1} - \bft_i$
      \STATE $d\bfy \gets \Phi(f,Q_{16}(\bfy), \bft_i, h_i, Q_{16}(\bftheta))$
        \COMMENT{evaluate increment in low precision}
      \STATE $\bfy \gets \bfy + h_i Q_{32}(d\bfy)$
        \COMMENT{accumulate in high precision}
      \STATE $\bfy_{i+1} \gets Q_{16}(\bfy)$
        \COMMENT{store state in low precision}
    \ENDFOR
    \STATE \textbf{return} $\{\,\bfy_i\}_{i=0}^N$
  \end{algorithmic}
\end{algorithm}

Using the output from the forward propagation, we approximate an objective functional of the form
\begin{equation*}
    \mathcal{L}(t,\bfx,\bftheta) = \int_0^T R(t,\bfy(t),\bftheta(t)) dt + C(\bfy(T))
\end{equation*}
with running cost $R$ and terminal cost $C$.
We use a quadrature, leading to the discrete objective function
\begin{equation}
L(\{\bfy_i(\bft,\bfx,\bftheta)\}_{i=0}^N, \bftheta, \bft) = \sum_{i=0}^N w_i R(\bft_i,\bfy_i,\bftheta(\bft_i)) + C(\bfy_N).
\end{equation}
In the backward pass, we differentiate the objective function with respect to the intermediate states of the ODE, the neural network weights, and the time steps by reversing the time integration.
We initialize $\bfa = Q_{32}\left(\frac{\partial C}{\partial \bfy_N}\right)$,  $\bfg =  Q_{32}\left(\frac{\partial L}{\partial \bftheta}\right)
\in \R^p,$ and $\bft' =  Q_{32}(\frac{\partial L}{\partial \bft})
\in \R^{N+1}$.

Our dynamic adjoint scaling scheme, which can be seen as an adaptation of the layer-wise scaling in~\cite{ZhaoEtAl2019} to the ODE setting, scales the adjoint variable so that the  range of the low-precision system is maximally exploited.
To minimize the risk of underflow, we initialize the adjoint scale to $S_N = 2^{\left\lfloor{-\log_2(\ulow \|\bfa\|)}\right\rfloor}$ to ensure that $\|S_N \bfa \| \approx \frac{1}{\ulow}$. Here, $\ulow$ is the unit roundoff of the low precision, $\|\cdot\|$ is the infinity norm of a vector. For the \texttt{float16} format it is $\ulow := 2^{-11}$ and for \texttt{bfloat16} it is $\ulow=2^{-8}$.
To avoid additional errors from rounding the mantissa, we round the loss scale to a power of two so that the scaling only affects the exponent.

For $i = N-1, N-2, \ldots, 0$, we use automatic differentiation to compute the vector-Jacobian products (vjp) with $\bfJ_\Phi = \left[\left(\frac{\partial \Phi}{\partial \bfy}\right)^\top, \left(\frac{\partial \Phi}{\partial \bft}\right)^\top , \left(\frac{\partial \Phi}{\partial h}\right)^\top, \left(\frac{\partial \Phi}{\partial \bftheta}\right)^\top \right]$ in low precision to compute the updates
\begin{equation}\label{eq:vjp}
 [d\bfa^\top, d\bft^\top, dh^\top, d\bfg^\top] = Q_{16}( S_{i} \bfa)^\top \bfJ_\Phi(f, \bfy_i, \bft_i, h_i, \bftheta).
\end{equation}
Since all these derivatives are along the direction $\bfa$, the computations are combined into one call of the automatic differentiation and the computational graph is built only once by recomputing the increment function using the $i$th state in low precision.

If any of the steps, $d\bfa $, $d\bft $, $dh, d\bfg$ contains any overflow or not-a-number (NaN), we halve the scale factor setting $S_i \gets S_i/2$  and repeat~\eqref{eq:vjp} until all quantities are finite.
These computations re-use the computational graph and do not require additional recomputation of the increment function.
When dynamic scaling is not used, our implementation either performs no overflow checking for optimal performance (\texttt{float32}, \texttt{bfloat16}) or returns infinity gradients for compatibility with PyTorch's GradScaler (\texttt{float16} without dynamic scaling).

After successful computation of the updates, we accumulate the derivative information in high precision by setting
\begin{align}
    {\bft}'_i &\gets \bft'_i + (h_i/S_i) Q_{32}(d\bft_i - dh_i) -  Q_{32} \left(\Phi(f,\bfy_i,\bft_i,h_i,\bftheta)^\top \bfa_{}\right) \\
    {\bft}'_{i+1} &\gets \bft'_{i+1} + (h_i/S_i) dh_i +   Q_{32} \left(\Phi(f,\bfy_i,\bft_i,h_i,\bftheta)^\top \bfa_{}\right)\\
    {\bfa} &\gets \bfa + w_i Q_{32}\left( \frac{\partial R}{\partial \bfy}\right) - (h_i/S_i) Q_{32}(d\bfa)\\
    {\bfg} &\gets \bfg + (h_i/S_i) Q_{32}(d\bfg).
\end{align}
If no re-scaling was needed in the Jacobian computation,  all these quantities are finite, and the adjoint variable is small (i.e.,  $\|\bfa\| \ulow \leq \frac{1}{2}$) we  double the scale factor for the next time step.
This completes the step, and we continue with the $(i-1)$th step.

This process yields the total derivatives of the loss, $\frac{dL}{d\bfx}=\bfa, \frac{dL}{d\bftheta} = \bfg, \frac{dL}{d\bft}=\bft'$, and is summarized in~\cref{alg:backward-pass}.

\begin{algorithm}[t]
  \caption{Mixed Precision Backward Pass with Dynamic Adjoint Scaling}
  \label{alg:backward-pass}
  \begin{algorithmic}[1]
    \REQUIRE state trajectory $\{\bfy_i\}_{i=0}^N$, loss function $L(\{\bfy_i(\bft,\bfx,\bftheta)\}_{i=0}^N,\bftheta, \bft)$, parameters $\bftheta \in\R^p$, maximum number of attempts $k_{\rm max}$
    \ENSURE total derivatives $\frac{d {L}}{d \bfx}, \frac{d {L}}{d \bftheta}, \frac{d {L}}{d \bft}$
    \STATE $\bfa \gets Q_{32}\left(\frac{\partial {L}}{\partial \bfy_N}\right), \bfg = Q_{32}\left(\frac{\partial L}{\partial \bftheta}\right), \bft' = Q_{32}\left(\frac{\partial L}{\partial \bft}\right)$ \COMMENT{partials in high precision}
    \STATE $S_N \gets  2^{\left\lfloor{-\log_2(\ulow \|\bfa\|)}\right\rfloor}$ \COMMENT{initialize scaling factor}
    \STATE $\bftheta \gets Q_{16}(\bftheta)$ \COMMENT{weights in low precision}
    \FOR{$i = N-1$ \TO $0$}
      \STATE $h_i  \gets \bft_{i+1} - \bft_i$
      \STATE $d\bfy \gets \Phi(f, \bfy_i, \bft_i, h_i, \bftheta)$
        \COMMENT{re-compute with low precision}
      \STATE $S_i \gets S_{i+1}$ \COMMENT{seed scale from previous step}
      \FOR{$k=1,\ldots, k_{\rm max}$}
        \STATE $d\bfa^\top, d\bft_i, dh_i, d\bfg^\top \gets  Q_{16}(S_i \bfa)^\top \bfJ_\Phi (f, \bfy_i, \bft_i, h_i, \bftheta)$
        \IF {$d\bfa, d\bfg, d\bft_i, dh_i$ all finite }
            \STATE \textbf{break} \COMMENT{accept sensitivities}
        \ENDIF
        \STATE $S_i \gets S_i/2$ \COMMENT{halve scale on overflow and retry}
      \ENDFOR
      \IF {$d\bfa, d\bfg, d\bft_i, dh_i$ not all finite}
        \STATE \textbf{error}: dynamic adjoint scaling failed at step $i$
      \ENDIF
      \STATE $\bft'_i \gets \bft'_i + (h_i/S_i) Q_{32}(d\bft_i - dh_i) - Q_{32}(\Phi(f,\bfy_i, \bft_i, h_i, \bftheta)^\top \bfa_{i+1}) $\COMMENT{high-precision}
      \STATE $\bft_{i+1}' \gets \bft'_{i+1} + (h_i/S_i) dh_i + Q_{32}(\Phi(f,\bfy_i, \bft_i, h_i, \bftheta)^\top \bfa_{i+1})$     \COMMENT{high-precision update}
      \STATE $\bfa \gets \bfa + w_i Q_{32} \left(\frac{\partial R}{\partial \bfy} \right) - (h_i/S_i) Q_{32}(d\bfa)$ \COMMENT{high-precision update}
      \STATE $ \bfg \gets \bfg + (h_i/S_i) Q_{32}(d\bfg)$ \COMMENT{high-precision update}
      \IF {$k=1$ (no scaling was needed) \AND $\|\bfa\|\leq \frac{1}{2\ulow }$}
        \STATE $S_i \gets 2 S_i$
      \ENDIF
    \ENDFOR
    \STATE \textbf{return} $\frac{d L}{d \bfx} = \bfa,  \frac{d L}{d \bftheta}= \bfg, \frac{d L}{d\bft} = \bft'$
  \end{algorithmic}
\end{algorithm}

\begin{example}
\label{ex:test_adjoint_scaling}
To assess the effectiveness of our custom backpropagation and dynamic adjoint scaling, we consider the following test problem
\begin{equation}
    y'(t) = -\lambda(t) y(t), \quad \text{where } \quad \lambda(t) = \bftheta_1 t^2 + \bftheta_2 t +\bftheta_3, \quad \text{ and } y(0)=x
\end{equation}
whose analytical solution is
\begin{equation}
    y(t) = x  \exp\left(-  \left(\frac{\bftheta_1}{3} t^3  + \frac{\bftheta_2}{2} t^2 + \bftheta_3 t\right)\right).
\end{equation}
We select the parameters of this problem such that $|y(t)|$ and $|y'(t)|$ span nearly the entire range of the normal numbers in \texttt{float16}, which is $[2^{-14}, 65504]$.
We use the terminal time $T=2.65$, the weights $\bftheta  = (8, -11, 2^{-16})$, and the initial state $x=\frac{65504}{180}$.
We use a 4th-order Runge-Kutta method with 400 equidistant time steps.
As can be seen in \cref{fig:test_adjoint_plot}, the analytical solution and its derivative vary by orders of magnitude and remain within the normal range of \texttt{float16}.
The low-precision approximation (indicated by a dashed red line) fits the analytical solution closely.

In~\cref{tab:test_adjoint_err}, we report the relative errors between the true quantities and the numerical approximations of the terminal state and the derivatives of the loss function $L(\{y_i\}_{i=1}^N, \bftheta, \bft)=\frac12 \|y_N\|_2^2$ with respect to the initial state and the components of $\bftheta$ for the precision of \texttt{float32} and \texttt{float16} and with and without scaling. As expected, scaling does not affect the accuracy of high-precision computations but is crucial in \texttt{float16}. Here, derivative computations underflow without scaling, whereas our dynamic scaling heuristic yields gradient approximations nearly as accurate as the forward solve.
It is important to note that simply scaling the loss function to avoid underflow would cause overflow at intermediate time steps in this example.

\begin{figure}[t]
  \centering
    \centering
    \includegraphics[width=.7\linewidth]{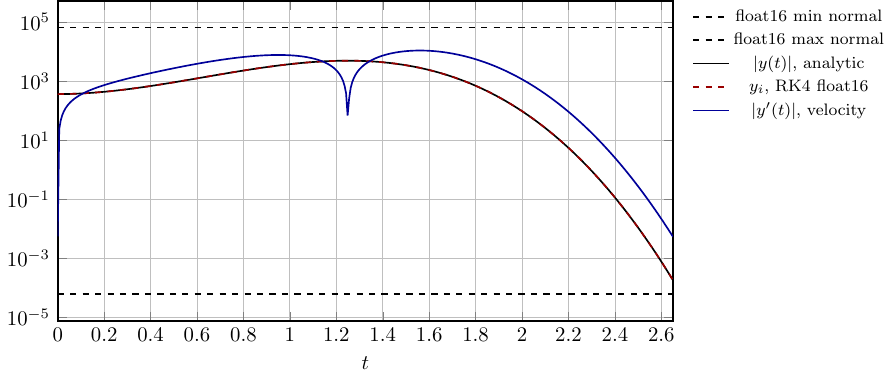}
    \caption{Log-scale plot of $|y(t)|$ (black solid line) and $|y'(t)|$ (blue solid line) used in \cref{ex:test_adjoint_scaling} and the numerical approximation of the solution using RK4 with 400 equidistant time steps in \texttt{float16}. The horizontal dashed lines indicate the range of the normal numbers in \texttt{float16}. In this set up, the ODE almost exhausts the range of \texttt{float16}, complicating the derivative computation.}
    \label{fig:test_adjoint_plot}
\end{figure}

\begin{table}[ht]
  \centering
    \vspace{1em}
    \begin{tabular}{@{}llccccc@{}}
      \toprule
      \textbf{dtype} & \textbf{scaling} & {$\mathrm{RE}\ y(T)$} & {$\mathrm{RE}\ \partial_{y_0}$} & {$\mathrm{RE}\ \partial_{\theta_1}$} & {$\mathrm{RE}\ \partial_{\theta_2}$} & {$\mathrm{RE}\ \partial_{\theta_3}$} \\
      \midrule
      \texttt{float32} & none     & 7.01e-5 & 1.40e-4 & 1.25e-4 & 1.30e-4 & 1.34e-4 \\
      \texttt{float32} & dynamic      & 7.01e-5 & 1.60e-4 & 1.28e-4 & 1.35e-4 & 1.45e-4 \\
      \texttt{float16} & none     & 3.67e-3 & 1.99e+3 & 1.00e+0 & 1.00e+0 & 1.00e+0 \\
      \texttt{float16} & dynamic      & 3.67e-3 & 5.89e-3 & 6.05e-3 & 5.96e-3 & 5.88e-3 \\
      \texttt{bfloat16} & none & 3.65e-2 & 4.50e-2 & 5.24e-2 & 4.96e-2 & 4.73e-2 \\
      \texttt{bfloat16} & dynamic & 3.65e-2 & 4.49e-2 & 5.24e-2 & 4.95e-2 & 4.73e-2 \\
      \bottomrule
    \end{tabular}
  \caption{Comparing the relative errors ($\mathrm{RE}$) of numerical approximations of the terminal state and the derivatives of $L(\{y_i\}_{i=1}^N, \bftheta, \bft) = \frac12 \|y_N\|_2^2$ with respect to the initial state and weights for different precisions and scaling strategies in \cref{ex:test_adjoint_scaling}. The table demonstrates the necessity and effectiveness of dynamic adjoint scaling in \texttt{float16}, which leads to gradient approximations with accuracy comparable to the forward solve. As expected, the gradient scaling has little effect on \texttt{float32} and \texttt{bfloat16} as they have the same range.}
  \label{tab:test_adjoint_err}

\end{table}

\end{example}

\section{Roundoff Error Analysis}\label{sec:roundoff}
We show that, for sufficiently regular ODEs, the relative errors of the forward solution, adjoint variables, and gradients between the exact time integrator and our mixed precision scheme are of the order of the low precision unit roundoff and do not grow uncontrolled with the number of time steps.

For brevity, we limit our analysis to the error introduced by finite-precision arithmetic but note that our result can be used to bound the overall error between the numerical and true solution.
To this end, we note that the total error comprises the discretization error and the mixed precision error.
If the latter can be bounded independently of the step size, this means that our mixed precision scheme is expected to show the same order of convergence until the discretization error reaches that level.

In our analysis, we consider a fixed $\bftheta$ and equidistant time grid with step size $h$ to simplify our notation; that is, we write the update mapping as $\Phi=\Phi(\bfy) $ suppressing the dependence on $f$, $h$ and $\bftheta$. 
Unless otherwise noted, products and divisions of two vectors in this section are to be understood component-wise.
We denote exact quantities (computed in infinite precision) without decoration (for example, $\Phi, \bfy_i, \bfa_i$) and rounded quantities (computed in finite precision) with a tilde (for example, $\tilde{\Phi}, \tilde{\bfy}_i, \tilde{\bfa}_i$), following the convention of Higham~\cite{Higham2002}.
We denote error vectors as $\bfe_{\bfx} = \bfx - \tilde{\bfx}$ (exact minus rounded) and, for vectors $\bfx$ with non-zero components, relative error vectors as $\bfdelta_\bfx = \frac{\bfe_{\bfx}}{\bfx}$.
The vectors $\bfr_L$ and $\bfr_H$ with $\|\bfr_L\| \leq \ulow$ and $\|\bfr_H\| \leq \uhigh$ denote the relative quantization errors and $\ulow$ and $\uhigh$ are the unit roundoffs of the low and high precision, respectively.
We assume round-to-nearest mode, the IEEE 754 default rounding mode.
As common, we measure all errors in the infinity norm of a vector and to ease notation, we drop the subscript on norms.

The quantization errors satisfy
\begin{equation*}
    Q_{16}(\bfx) = \tilde{\bfx} = (\bfone + \bfr_{16}) \bfx,
    \quad \text{ for } \quad \|\bfr_{16}\| \leq \ulow,
\end{equation*}
where $\bfone$ is a vector of all ones and $\tilde{\bfx}$ denotes the rounded value.
We further assume that $Q_{32}(\bfx)$ introduces no error when converting from low to high precision, as low-precision values are exactly representable in high precision.
However, rounding the result of an addition in high precision introduces an error
\begin{equation*}
   Q_{32}(\bfx+\bfy) = (\bfone+\bfr_{32}) (\bfx+\bfy) \quad \text{with} \quad \|\bfr_{32}\| \leq \uhigh.
\end{equation*}

To simplify our notation, we ignore the running costs as they can be added to the state vector.

We denote with $\DPhiexact$ and $\nablaCexact$  the Jacobian of $\Phiexact$ and gradient of $\Cexact$ in exact arithmetic, and with
$\DPhirounded(\bfy)$ and $\nablaCreounded$  their finite precision counterparts obtained using automatic differentiation and finite precision arithmetic.
\begin{assumption}\label{assumption}\
\begin{enumerate}
    \item {\bf Smoothness and boundedness}: The increment function $\Phiexact$ is bounded and has bounded first and second derivatives, and we denote with $L_{\Phiexact}$ and $L_{\DPhiexact}$ constants such that
    \begin{equation}\label{eq:bound_DPhi}
    \|\DPhiexact(\bfy)\| \leq L_{\Phiexact}, \quad \|D^2\Phiexact(\bfy)\| \leq L_{\DPhiexact}, \quad \forall \bfy \in \R^d, \quad i=1,\ldots,N,
    \end{equation}
    and similarly for its finite precision implementation we assume there is a constant $L_{\Phirounded}$ such that
\begin{equation}\label{eq:bound_DPhiFP}
    \|\DPhirounded(\bfy)\| \leq L_{\Phirounded}, \quad \text{for all } \bfy \in \R^d, \quad i=1,\ldots,N.
    \end{equation}

This constant depends on the number and nature of floating point operations performed in each evaluation.

The cost function $\Cexact$ is differentiable and has bounded derivatives, with Lipschitz constant $L_{\nabla C}$, for all $\bfx, \bfy \in \R^d$,
    \begin{equation}\label{eq:LipschitzNablaC}
        \|\nablaCexact(\bfx) - \nablaCexact(\bfy)\| \leq L_{\nabla C} \|\bfx - \bfy\|.
    \end{equation}

    \item {\bf Behavior along the states sequences}: The components of the states $\bfy_1, \bfy_2, \ldots, \bfy_N$     are bounded above by $M_\bfy$.
    If components are zero or near-zero, one could add a small shift $m_\bfy$ to avoid division by zero when deriving relative errors. We also assume there are constants $G_{\bfy}$  and $G_{\bfa}$ such that
    $$G_{\bfy}:=\max_i \left\|  \frac{ {\Phirounded_i(\bfy_i)}}{\bfyrounded{}_{i+1}}\right\|,\qquad
    G_{\bfa}:=\max_i \left\| \frac{d\bfarounded{}_i}{\bfarounded{}_{i-1}} \right\|.$$
    And there are constants
    $\mu_{\bfy}$ and $\mu_{\bfa}$ such that $$\mu_{\bfy}:=\max_{0\le i\le N-1}\|\bfyrounded{}_i\|\,\|\bfyrounded{}_{i+1}^{-1}\|,\qquad
       \mu_{\bfa}:=\max_{0\le i\le N-1}\|\bfarounded{}_{i+1}\|\,\|\bfarounded{}_{i}^{-1}\|.$$

 There is a constant $m_{\nabla C}$ such that
 $   m_{\nabla C}:=\min_i \|\nabla C(\bfy_i)\|.$

\item  {\bf Relative error bounds in terms of the precision}: There exist constants $M_\Phi$, $M_{D\Phi}$, and $M_{\nabla C}$ and the following bounds of the relative error of low-precision evaluations. Along $\{ \bfy_i \}_{i=0}^N$
we have
\begin{align}
\bfdelta_{\Phi(\bfy_i)}:=  \frac{\Phiexact(\bfy_i)-\Phirounded(\bfy_i)}{\Phirounded(\bfy_i)},\qquad \|\bfdelta_{\Phi(\bfy_i)}\| &\leq M_\Phi \ulow \label{eq:bound_eval},   \\
\|\DPhiexact(\bfy_i)^\top-\DPhirounded(\bfy_i)^\top \|\leq M_{D\Phi} \ulow \|\DPhirounded(\bfy_i)^\top \|,  \label{eq:bounded_DPhi}\\[0.15cm]
\bfdelta_{\nabla C}:=\frac{\nablaCexact(\bfy_i)-\nablaCreounded(\bfy_i)}{\nablaCreounded(\bfy_i)}, \qquad \left \|\bfdelta_{\nabla C} \right\| &\leq M_{\nabla C} \ulow \label{eq:round_err_C}.
\end{align}

\end{enumerate}

\end{assumption}

The goal of this section is to prove the following error estimates.
\begin{theorem}[Main result]\label{theo:main_result}
    Under \cref{assumption} and with $\|\bfdelta_{\bfy_0}\|=\mathcal{O}(\ulow)+ \mathcal{O}(\uhigh/h)$, our mixed precision time integrator and custom backward satisfy:
    \begin{enumerate}
        \item accuracy of the state: $\|\bfdelta_{\bfy_i}\| = \mathcal{O}(\ulow)+ \mathcal{O}(\uhigh/h)$;
        \item accuracy of the adjoint: $\|\bfdelta_{\bfa_i}\| = \mathcal{O}(\ulow)+ \mathcal{O}(\uhigh/h)$;
        \item accuracy of the gradient: $\|\bfdelta_{\mathbf{g}}\| = \mathcal{O}(\ulow) + \mathcal{O}(\uhigh/h)$.
    \end{enumerate}
\end{theorem}
The theorem has practical implications for choosing a step size. Balancing the two bounds $\mathcal{O}(\ulow)$ and $\mathcal{O}(\uhigh/h)$ yields a crossover at $h\sim \uhigh/\ulow$. For $h\gg \uhigh/\ulow$, the $\mathcal{O}(\ulow)$ term dominates and the mixed precision scheme works well. For $h\ll \uhigh/\ulow$, the accumulation term $\mathcal{O}(\uhigh/h)$ dominates and limits the performance of mixed precision training. In between, the performance depends on the constants.

\begin{lemma}[Forward error]\label{lemm:fwd_error}
    Under~\cref{assumption}, for all $i=1,\ldots,N$,  the relative forward propagation error $\bfdelta_{\bfy_i} = \frac{\bfyexact{i} - \bfyrounded{i}}{\bfyrounded{i}}$ satisfies
    \begin{equation}\label{eq:rel_err_y}
        \|\bfdelta_{\bfy_i}\| \leq  \|\bfdelta_{\bfy_0}\| \exp(\bft_i K_1) + (\exp(\bft_i K_1)-1) \left( K_2 \ulow  + K_3 \frac{\uhigh}{h} \right),
    \end{equation}
    for constants $K_1, K_2, K_3$ independent of $h$.
\end{lemma}
\begin{proof}
    The $i$th step of the exact and mixed precision scheme are, respectively,
\begin{align*}
\bfyexact{i+1} & = \bfyexact{i} + h \Phiexact(\bfyexact{i})\\
(\bfone+\bfr_H^{i+1})\bfyrounded{i+1} & = (\bfyrounded{i} + h \Phirounded(\bfyrounded{i}))
\end{align*}
Subtracting the two gives the error
\begin{align}
  \bfyexact{i+1} -\bfyrounded{i+1} &= \left(\bfyexact{i}-\bfyrounded{i}  \right) + h \left(\Phiexact(\bfyexact{i})-\Phirounded(\bfyrounded{i})  \right) + \bfr_H^{i+1}  \bfyrounded{i+1}.
\label{eq:error_step}
\end{align}
Using Taylor's theorem, the error in the step computation can be expressed as
\begin{align*}
    \Phiexact(\bfyexact{i})-\Phirounded(\bfyrounded{i})  & = \Phiexact(\bfyexact{i}) -\Phiexact(\bfyrounded{i}) + \Phiexact(\bfyrounded{i})
    -\Phirounded(\bfyrounded{i})    \\
    & = \DPhiexact(\bfxrounded{i})\bfe_{\bfy_i}+\bfe_{\Phi(\bfy_i)}
\end{align*}
 where $\DPhiexact$ is the Jacobian of the increment function $\Phiexact$, the $j$th component of $\bfxrounded{i}$ is between the corresponding components of $\bfyrounded{i}$ and $\bfyexact{i}$,  and $\bfe_{\Phi(\bfy_i)} = \Phiexact(\bfyrounded{i})-\Phirounded(\bfyrounded{i}) $.
Substituting this into \eqref{eq:error_step} gives
\begin{align*}
    \bfe_{\bfy_{i+1}} &= (\bfI + h \DPhiexact(\bfxrounded{i})) \bfe_{\bfy_i} + h \bfe_{\Phi(\bfy_i)} + \bfr_H^{i+1}  \bfyrounded{i+1} ,
\end{align*}
where $\bfI$ denotes the identity matrix.
At this point, we can already see that the error due to the low-precision computation of the increment function is scaled by $h$ while the roundoff error from adding the increment is not, which can cause it to dominate the error propagation as the total number of time steps, $N$, increases.
To get a relative error, we divide component-wise by  $\bfyrounded{i+1}$ and simplify
\begin{align*}
    \bfdelta_{\bfy_{i+1}}
    & = \bfyrounded{i+1}^{-1}(\bfI + h \DPhiexact(\bfxrounded{i})) \bfyrounded{i}  \bfdelta_{\bfy_i}  + h   \frac{ \bfe_{\Phi(\bfy_i)}}{\bfyrounded{i+1}} +\bfr_H^{i+1},
\end{align*}
and obtain the bound
\begin{align*}
\|\bfdelta_{\bfy_{i+1}}\|
& \leq \left\| \bfyrounded{i+1}^{-1}\left(
\bfI + h \DPhiexact(\bfxrounded{i}) \right)\bfyrounded{i} \right\|   \left\|\bfdelta_{\bfy_i}\right\|  + h \left\|  \frac{ \bfe_{\Phi(\bfy_i)}}{\bfyrounded{i+1}} \right\| + \|\bfr_H^{i+1}\|.
\end{align*}

To obtain a step-dependent bound on the first term, we use~\eqref{eq:bound_DPhi} and obtain
\begin{equation*}
    \left\| \bfyrounded{i+1}^{-1}\left(
\bfI + h \DPhiexact(\bfyexact{i}) \right)\bfyrounded{i} \right\|\le 1+h(G_{\bfy}+L_{\Phiexact}\mu_{\bfy})+\uhigh,
\end{equation*}
where we have used that
$$\bfyrounded{i+1}^{-1}\bfyrounded{i}=\bfyrounded{i+1}^{-1}((\bfone-\bfr_H^{i+1})\bfyrounded{i+1}-h\Phirounded(\bfyrounded{i}))=\bfone - \bfr_H^{i+1}-h\bfyrounded{i+1}^{-1}\Phirounded(\bfyrounded{i}),$$
and therefore $\|\bfyrounded{i+1}^{-1}\bfyrounded{i}\|\le 1+hG_{\bfy}+\uhigh$ upon discarding higher order terms in $h$.

To obtain a step-independent bound on the second term, we use~\eqref{eq:bound_eval} and obtain
\begin{align*}
\left\|  \frac{ \bfe_{\Phi(\bfy_i)}}{\bfyrounded{i+1}} \right\|= \left\| \frac{ {\Phirounded(\bfyrounded{i})}}{\bfyrounded{i+1}}\bfdelta_{\Phi(\bfy_i)} \right\|  \leq \left\| \frac{ {\Phirounded(\bfyrounded{i})}}{\bfyrounded{i+1}}\right\| \left\|\bfdelta_{\Phi(\bfy_i)} \right\|  \leq G_{\bfy} M_\Phi  \ulow.
\end{align*}
This gives the recursive error bound
\begin{equation*}
\|\bfdelta_{\bfy_{i+1}}\| \leq \left(1 + h\alpha \right) \|\bfdelta_{\bfy_i}\| + (h\beta \ulow +\uhigh),
\end{equation*}
with $\alpha = \mu_{\bfy}L_{\Phiexact} + G_\bfy$ and $\beta = M_\Phi G_{\bfy} $ and applying the linear recursion formula~\cite[Lemma 7.2.2.2]{stoer2002introduction}, we get finally
\begin{align*}
\|\bfdelta_{\bfy_{i}}\| &\leq \exp(\bft_i \alpha )\|\bfdelta_{\bfy_0}\| +  \frac{\exp(\bft_i\alpha)-1}{\alpha} \left( \beta \ulow + \frac{\uhigh}{ h }\right).
\end{align*}
\end{proof}

\begin{lemma}[Adjoint error]\label{lemm:bwd_error}
    Under~\cref{assumption}, the relative adjoint error $\bfdelta_{\bfa_i} = \frac{\bfaexact{i} - \bfarounded{i}}{\bfarounded{i}}$ satisfies
    \begin{equation*}
    \|\bfdelta_{\bfa_i}\| \leq \exp((T-\bft_i) K_4) \|\bfdelta_{\bfa_N}\| + K_5\|\delta_{\bfy_0}\|+ \left( K_6 \ulow + K_7 \frac{\uhigh }{h} \right),
    \end{equation*}
    for constants $K_4, K_5, K_6, K_7$ independent of $h$.
\end{lemma}
\begin{proof}
    The relative error in the terminal adjoint state is equal to the relative error of the low-precision evaluation of $\nabla C$, that is, $\bfdelta_{\bfa_N} = \bfdelta_{\nabla C}$.
    For the $i<N$th step, subtracting the low-precision  adjoint step, $\bfarounded{i-1}$ from exact counterpart, $\bfaexact{i-1}$, we get the error propagation
\begin{equation*}
\bfe_{\bfa_{i-1}} =  \bfe_{\bfa_i} - h \bfe_{d\bfa_i} + \bfr_H^{i-1}  \bfarounded{i-1},
\end{equation*}
where the error in the increment is given by
\begin{align*}
\bfe_{d\bfa_i} & = d\bfaexact{}_i-d\bfarounded{}_i  =\DPhiexact(\bfyexact{i-1})^\top \bfaexact{i} -  \DPhirounded(\bfyrounded{i-1})^\top \bfarounded{i}  \\
 & = \DPhiexact(\bfyexact{i-1})^\top \bfe_{\bfa_i}+(\DPhiexact(\bfyexact{i-1})-\DPhirounded(\bfyrounded{i-1}) )^\top \bfarounded{i}.
\end{align*}
This gives the error propagation
\begin{equation*}
\bfe_{\bfa_{i-1}} = (\bfI - h \DPhiexact(\bfyexact{i-1})^\top)\bfe_{\bfa_i} - h \left(\DPhiexact(\bfyexact{i-1})- \DPhirounded(\bfyrounded{i-1}) \right)^\top \bfarounded{i}  + \bfr_H^{i-1}  \bfarounded{i-1}.
\end{equation*}
By dividing component-wise by $\bfarounded{i-1}$ and simplifying
\begin{align*}
\bfdelta_{\bfa_{i-1}} &= \bfarounded{i-1}^{-1}(\bfI - h \DPhiexact(\bfyexact{i-1})^\top)\bfarounded{i}\cdot \bfdelta_{\bfa_i} - h \bfarounded{i-1}^{-1}\left( \DPhiexact(\bfyexact{i-1})-\DPhirounded(\bfyrounded{i-1})  \right)^\top \bfarounded{i} + \bfr_H^{i-1}.
\end{align*}
Using the bounds on the Jacobian, we see that
\begin{equation*}
   \| \bfarounded{i-1}^{-1}(\bfI - h \DPhiexact(\bfyexact{i-1})^\top)\bfarounded{i}\|\le 1+h(G_{\bfa}+\mu_{\bfa}L_{\Phiexact})+\uhigh
    \end{equation*}
where we have used
$$ \bfarounded{i-1}^{-1}\bfarounded{i}=\bfarounded{i-1}^{-1}((\bfone-{\bf r}_H^{i-1})\bfarounded{i-1}+hd\bfa_i)=\bfone -{\bf r}_H^{i-1}+h\frac{d\bfa_i}{\bfarounded{i-1}}, \qquad \|\bfarounded{i-1}^{-1}\bfarounded{i}\|\le 1+hG_{\bfa} +\uhigh.$$
Taking the norm of both sides and using the triangle inequality, and dropping terms $\mathcal{O}(h\uhigh)$, we obtain
\begin{equation*}
\|\bfdelta_{\bfa_{i-1}}\| \leq  (1 +h (\mu_{\bfa}L_{\Phiexact} +G_{\bfa}))   \|\bfdelta_{\bfa_i}\| + h \mu_{\bfa}\left\| \DPhiexact(\bfyexact{i-1})^\top-\DPhirounded(\bfyrounded{i-1})^\top  \right\|  + \uhigh.
\end{equation*}
To obtain a step-independent bound on the second term, we compute
\begin{align*}
\|\DPhiexact(\bfyexact{i-1})^\top-\DPhirounded(\bfyrounded{i-1})^\top  \|  & = \|(\DPhiexact(\bfyexact{i-1})- \DPhiexact(\bfyrounded{i-1}))^\top  +  (\DPhiexact(\bfyrounded{i-1})-\DPhirounded(\bfyrounded{i-1}))^\top   \|\\
&\leq \|(\DPhiexact(\bfyrounded{i-1}) -\DPhiexact(\bfyexact{i-1}))^\top \|+\|(\DPhiexact(\bfyrounded{i-1})-\DPhirounded(\bfyrounded{i-1}))^\top \| \\
& \leq L_{\DPhiexact} \|\bfdelta_{\bfy_{i-1}}  \bfyrounded{i-1}\|+\ulow   M_{D\Phi} \|\DPhirounded(\bfyrounded{i-1})^\top\|   \\
& \leq L_{\DPhiexact} M_\bfy \|\bfdelta_{\bfy_{i-1}}\|+\ulow  M_{D\Phi} L_{\Phirounded},
\end{align*}
where we used the Lipschitz continuity of $\DPhiexact$, \eqref{eq:bounded_DPhi}, the upper bound of $\{ \|\bfy_i \| \}_{i=0}^N$ and the boundedness of $\DPhirounded$ \eqref{eq:bound_DPhiFP}.
Dropping higher order terms, we obtain the recursive error bound
$$  \|\bfdelta_{\bfa_{i-1}}\| \leq (1+h\gamma)\|\bfdelta_{\bfa_i}\|+h\theta \|\delta_{\bfy_{i-1}}\|+(h\sigma \ulow+\uhigh), $$
with $\gamma=(\mu_{\bfa}L_{\Phiexact} + G_{\bfa})$, $\tau=\mu_{\bfa}L_{\DPhiexact}M_{\bfy},$ $\sigma= \mu_{\bfa}M_{D\Phi}L_{\Phirounded}$.
Using~\eqref{eq:rel_err_y}, we can bound the contributions of the forward error as
\begin{equation*}
  \|\bfdelta_{\bfy_i}\| \leq  \|\bfdelta_{\bfy_0}\| \exp(\bft_N K_1) + \tilde{K}_2 \ulow  + \tilde{K}_3 \frac{\uhigh}{h}, \quad \forall i,
\end{equation*}
where $\tilde{K}_2 = (\exp(\bft_N K_1)-1) K_2$ and $\tilde{K}_3 = (\exp(\bft_N K_1)-1) K_3$.
This allows us to apply~\cite[Lemma 7.2.2.2]{stoer2002introduction} to
\begin{equation*}
    \|\bfdelta_{\bfa_{i-1}}\| \leq (1+h\gamma)\|\bfdelta_{\bfa_i}\|+h \left(\tau\|\bfdelta_{\bfy_0}\| \exp(\bft_N K_1) +(\tau \tilde{K}_2 + \sigma) \ulow  + (\tau \tilde{K}_3 +1) \frac{\uhigh}{h}\right),
\end{equation*}
which gives the final estimate
\begin{align*}
    \|\bfdelta_{\bfa_{i}}\| \leq & \exp((\bft_N-\bft_i)\gamma)  \|\bfdelta_{\bfa_{N}}\| \\
    &+ \frac{\exp((\bft_N-\bft_i) \gamma)-1}{ \gamma} \left( \frac{\tau\|\bfdelta_{\bfy_{0}}\| }{\exp(\bft_N K_1)} + (\tau \tilde{K}_2+\sigma) \ulow +(\tau\tilde{K}_3 +1) \frac{\uhigh}{h} \right).
\end{align*}
\end{proof}

\begin{lemma}[Gradient error]\label{lemm:grad_error}
    Under~\cref{assumption}, the relative error of $\bfg = \frac{dL}{d\bftheta}$ satisfies
    \begin{equation*}
    \|\bfdelta_{\bfg}\|  = \mathcal{O}(\ulow) + \mathcal{O}(\uhigh/h).
    \end{equation*}
\end{lemma}
\begin{proof}
    The gradient accumulations are
    \begin{equation*}
        \bfg = h \sum_{i=0}^{N-1}  Q_{32}\left( \frac{\partial \Phi}{\partial \bftheta}(Q_{16}(\bfy_i))^\top Q_{16}(\bfa_i) \right).
    \end{equation*}
    The sum contains $N=T/h$ terms, so we must make sure it does not grow unboundedly. Each term in the sum has an error of $h \mathcal{O}(\ulow)$ from the low-precision computations, and the high-precision summation adds $\mathcal{O}(\uhigh)$ for each of the $N$ terms. Since $N$ and $h$ cancel in the $\ulow$ term (giving $Nh \mathcal{O}(\ulow) = T\mathcal{O}(\ulow)$), while the high-precision errors accumulate as $N\mathcal{O}(\uhigh) = \frac{T}{h}\mathcal{O}(\uhigh)$, the result follows.
\end{proof}

We are finally ready to prove our main result~\cref{theo:main_result}.
\begin{proof}[Proof of~\cref{theo:main_result}]
    The first item follows from \cref{lemm:fwd_error}. The second item follows from \cref{lemm:bwd_error},
    noting that using \cref{eq:LipschitzNablaC} and \cref{eq:round_err_C} we have
\begin{align*}
\|\bfdelta_{\bfa_N}\| & \leq \left\|\frac{\nablaCexact(\bfyexact{N})-\nablaCreounded(\bfyrounded{N})}{\nablaCreounded(\bfyrounded{N})}  \right\| \\
& \leq \left\| \frac{\nablaCexact(\bfyexact{N}) - \nablaCexact(\bfyrounded{N})}{\nablaCreounded(\bfyrounded{N})}\right\|+ \left\|\frac{\nablaCexact(\bfyrounded{N})-\nablaCreounded(\bfyrounded{N})  }{\nablaCreounded(\bfyrounded{N})} \right\|\\
& \leq    \frac{L_{\nabla C}}{ m_{\nabla C }  }M_{\bfy}\|\delta_{\bfy_N}\|+M_{\nabla C} \ulow ,
\end{align*}
and using \cref{lemm:fwd_error} for bounding $\|\bfdelta_{\bfy_N}\|$.
The third item follows from \cref{lemm:grad_error}.
    We note the recurrent theme that the low-precision errors enter through function evaluations multiplied by $h$ and thus accumulate as $\mathcal{O}(\ulow)$, independent of $h$ while the high-precision errors from the state, adjoint, and gradient updates accumulate as $\mathcal{O}(\uhigh/h)$.
\end{proof}

\section{Implementation and Software}\label{sec:implementation}

We implement the mixed precision forward propagation of~\cref{alg:forward-pass} and the customized backward propagation with dynamic scaling of~\cref{alg:backward-pass} for neural ODEs in the PyTorch package \texttt{rampde}.
The package is designed to train neural ODEs using mixed precision as effectively as in single precision while reducing memory usage and time to solution.
We use PyTorch's built-in autocast functionality and hardware support to accelerate computations particularly on modern GPUs equipped with tensor compute units.

\begin{figure}[htbp]
\centering
\begin{lstlisting}[style=rampdepython]
#from torchdiffeq import odeint
from rampde import odeint
...
with autocast(device_type='cuda', dtype=dtype_low):
    y = odeint(func, y0, timesteps)
    loss = L(y)
    loss.backward()
\end{lstlisting}
\caption{Minimal example for switching from \texttt{torchdiffeq} to \texttt{rampde} requires only changing the import statement and wrapping the ODE solve, loss evaluation, and backward pass in a PyTorch autocast context that specifies the low-precision dtype.}
\label{fig:rampde_minimal}
\end{figure}
To allow users to experiment with \texttt{rampde} without excessive code changes, our software is designed to closely mirror that of \texttt{torchdiffeq}~\cite{torchdiffeq}, the perhaps most commonly used library for neural ODE training.
In most cases, using our package only requires changing an import and adding an autocast context, as illustrated in the minimal example in~\cref{fig:rampde_minimal}.

A key difference to \texttt{torchdiffeq} is that our package is limited to discretize-then-optimize schemes using fixed time steps.

The autocast context in the above example controls the mixed precision computation including casting variables and accelerating whitelisted operations; see the documentation of PyTorch's built-in autocast functionality~\cite{pytorch_autocast2024}.
Our current implementation is limited to CUDA-enabled GPUs, where we support both \texttt{float16} and \texttt{bfloat16} arithmetic.
Users can also experiment with \texttt{bfloat16} on some CPUs.
When mixed precision is not supported (e.g., on Apple Silicon/MPS devices), the code automatically falls back to single-precision arithmetic.

The \texttt{rampde} package can be installed via pip or conda. The source code is hosted on Github where we also provide detailed instructions, minimal examples, and list all required dependencies with explicit version constraints. We publish the code under the permissive MIT license.

The key component of our package is the mixed precision forward propagation in~\cref{alg:forward-pass} and the custom backpropagation with dynamic adjoint scaling in~\cref{alg:backward-pass}.
Our package is designed modularly to allow extending the available forward solvers (currently, forward Euler and RK4) and scaling (currently, no scaling for testing and dynamic adjoint scaling as default).

To achieve optimal performance across different precision formats and scaling requirements, we provide three backward pass implementations that share the same forward propagation.
The unscaled variant eliminates overflow checking and scaling infrastructure for minimal overhead and is the default for \texttt{float32} and \texttt{bfloat16}, where the wider range makes overflow unlikely.
The unscaled-safe variant adds exception handling for overflow protection and is the default for \texttt{float16} without dynamic scaling. It provides compatibility with PyTorch's standard GradScaler when overflow protection is needed without dynamic scaling overhead.
The dynamic variant implements the full scaling algorithm from~\cref{alg:backward-pass} with automatic scale adjustment for most robust \texttt{float16} training.
The appropriate solver variant is automatically selected based on the precision format and loss scaler type, removing the burden of manual selection from users while ensuring optimal performance for each configuration.

Although the package is primarily designed to support the development of mixed precision algorithms for neural ODE, our implementation is optimized to be deployed on recent GPUs with tensor compute units and we tested the code on NVIDIA RTX A6000 GPUs.
On our hardware, we find that the performance gains can be significant when the problem sizes are adequate for tensor cores and sufficiently large.

To verify the correctness of our code, we include a suite of unit tests. Using a linear ODE system, we test the approximation of solution and gradients. Using a nonlinear ODE, we test the correctness of our manual backward propagation.
Using~\cref{ex:test_adjoint_scaling}, we test the effectiveness of our adjoint scaling.
We also implement a regression test using a linear ODE to ensure that future updates to the code maintain some speed-ups for sufficiently large systems on our hardware.

One limitation of our package is that our custom backward pass casts all neural network weights and inputs to the specified low precision rather than using autocast.
This implementation was needed because of the limited interoperability of PyTorch's autocast and autograd features.
This means that all operations, including blacklisted ones in autocast, are performed in low-precision.
To avoid possible discrepancies between forward and backward pass, users need to add manual casting operations around blacklisted operations in their neural network model.
We closely monitor release notes of future versions of PyTorch and will update our code when the support for automatic mixed precision is sufficiently improved.

\section{Numerical Experiments}\label{sec:experiments}

In this section, we demonstrate the effectiveness of our mixed precision training framework for neural ODEs across three learning tasks of increasing computational complexity.
All experiments were conducted on NVIDIA RTX A6000 GPUs, which provide hardware acceleration for both \texttt{float16} and \texttt{bfloat16} arithmetic through tensor compute units.
We implemented our algorithms in PyTorch and provide the open-source \texttt{rampde} package for reproducibility.
We compare our implementation against the widely-used \texttt{torchdiffeq} package~\cite{torchdiffeq} across different floating-point precisions and experiment with different scaling strategies.

\subsection{Continuous Normalizing Flows}
\label{sub:cnf_experiment}

Continuous Normalizing Flows (CNFs) are a class of generative models that compute a transformation between a complex target distribution, $\pi_{\rm data}$, represented by samples and a simple base distribution, $\pi_1$, usually a standard normal through a continuous-time flow~\cite{ChenEtAl2019, GrathwohlEtAl2018}.
Using the change of variables formula, it can be seen that the log-likelihood of a data point $\bfx\sim\pi_{\rm data}$ under the flow is given by
\begin{equation}
    \log\pi_{\rm data}(\bfx) = \log\pi_1(\bfz(1)) -  \ell(1),
\end{equation}
where $\bfz(1)$ and $\ell(1)$ are the terminal state of the neural ODE system with joint state $\bfy(t) = [\bfz(t), \ell(t)]^\top$,
\begin{equation*}
\frac{d\bfy}{dt} = f(t, \bfy(t), \bftheta) = \begin{bmatrix} v(t, \bfz(t), \bftheta) \\ -\text{tr}\left(\frac{\partial v}{\partial \bfz}\right) \end{bmatrix}, \quad t \in (0,1], \quad \bfy(0) = \begin{bmatrix}
\bfx\\ 0
\end{bmatrix}.
\end{equation*}
Here, $\bfz(t) \in \R^d$ tracks the trajectories of the samples,  $\ell(t) \in \R$ computes the volume change, and  the velocity field $v: [0,1] \times \R^d \times \R^p \to \R^d$ is parameterized by a neural network.
To train the network weights, we minimize the expected negative log-likelihood over $\pi_{\rm data}$. To generate new samples, we reverse the flow with terminal state obtained from the standard normal distribution.
Assuming backward stability and sufficiently accurate time integration, the neural ODE enables tractable computations of the inverse of the flow and likelihood estimation. The high accuracy requirements in the ODE solver make CNF an interesting testbed problem for mixed precision training.

We evaluate our mixed precision framework on three standard 2D benchmark datasets (2-spirals~\cite{LangWitbrock1988}, 8-Gaussians, and checkerboard).
To enable a direct comparison with \texttt{torchdiffeq}, we modify the implementation of the two-dimensional toy problems used in~\cite{GrathwohlEtAl2018}.
We use the discretize-then-optimize approach and integrate the augmented ODE using RK4 with 128 fixed timesteps.
We use the Adam optimizer and train for 2000 iterations with batch size 1024 and learning rate 0.01 without weight decay.
We evaluate the performance using validation negative log-likelihood and assess the visual quality of generated samples.
Following \cite{GrathwohlEtAl2018}, we parameterize the vector field using a hypernetwork, that is,
\begin{equation*}
v(t, \bfz, \bftheta) = \frac{1}{w}\sum_{j=1}^{w} \bfU_j(t) \cdot \tanh(\bfW_j(t) \bfz + \bfb_j(t)),
\end{equation*}
where $\{\bfW_j(t), \bfb_j(t), \bfU_j(t)\}_{j=1}^w $  are computed using three fully connected layers with hidden dimension 32, generating $w=128$ sets of time-dependent weights for the velocity field with weights $\bftheta$.
The most significant modification we made to the code is replacing the stochastic trace estimation with a direct implementation of $\text{tr}(\partial v/\partial \bfz)$, because the automatic differentiation did not reliably compute derivatives in the precision set by autocast in our experiments.

\cref{fig:cnf_overview} visualizes the quality of generated samples across different precision formats and scaling strategies.
Both \texttt{torchdiffeq} and \texttt{rampde} produce visually indistinguishable samples to single precision when proper scaling is applied, demonstrating that mixed precision training leads to performance similar to that of single precision.
The quantitative results in \cref{tab:cnf_results} confirm this observation, with validation losses remaining consistent across precisions when scaling is enabled.
One indication that scaling is essential can be seen in the \texttt{torchdiffeq} results at \texttt{float16} where the training fails to reduce the validation loss.
The losses of all models start at approximately 4 for the 2-spirals dataset but for \texttt{torchdiffeq} without gradient scaling, the final loss is around 9.0, compared to 2.7 for the successful experiments.
We observe similar results for \texttt{torchdiffeq} without scaling on the other datasets.
Both gradient scaling and dynamic adjoint scaling effectively stabilize the \texttt{float16} training.

\begin{figure}
    \centering
    \includegraphics[width=.88\textwidth]{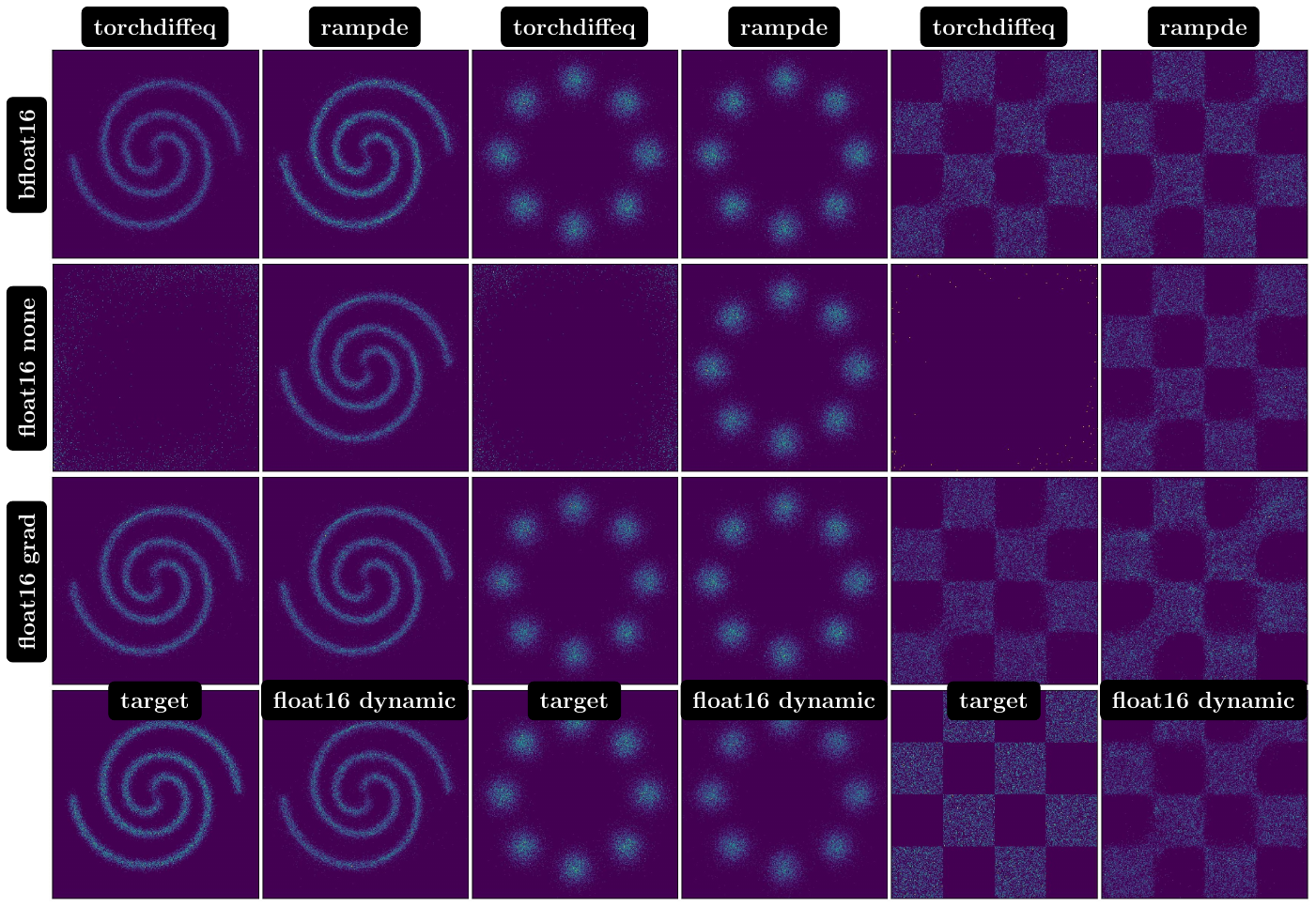}
    \caption{Comparison of continuous normalizing flow (CNF) sample quality in \cref{sub:cnf_experiment} across precision
  formats and ODE solvers on 2D datasets also used in~\cite{GrathwohlEtAl2018}. The columns contain   the datasets: 2-spirals (left pair), 8-Gaussians (center pair), and checkerboard (right pair) datasets. For each dataset, there are two columns, comparing torchdiffeq (left) and rampde (right). Rows demonstrate different precision and scaling configurations. The bottom row shows ground truth targets with \texttt{float16} dynamic scaling, which is only available in rampde. Across these tests, the \texttt{rampde} solver exhibits visually comparable sample quality even without loss scaling.}
    \label{fig:cnf_overview}
\end{figure}
\begin{table}[htbp]
\centering
\caption{Mixed precision training results for Continuous Normalizing Flows (CNFs) on the 2D datasets 2-spirals~\cite{LangWitbrock1988}, 8-Gaussians, and checkerboard. We test \texttt{torchdiffeq} and \texttt{rampde} with network evaluations in \texttt{float32}, \texttt{bfloat16}, and \texttt{float16}. For the latter, we test gradient scaling (Grad), no scaling (None), and dynamic scaling (Dyn). Both solvers achieve comparable results to single precision in \texttt{bfloat16} and \texttt{float16} with proper scaling. For \texttt{float16} without scaling, \texttt{torchdiffeq} performs 2000 iterations without NaN exceptions but the training loss increases. Since these are small-scale tests, we observe no speedup and only slight memory savings compared to single precision. As expected, \texttt{rampde} uses significantly less memory at the cost of more expensive backpropagation.
}
\label{tab:cnf_results}

\resizebox{\textwidth}{!}{%
\begin{tabular}{lccccccccc}
\toprule
& \multicolumn{2}{c}{\texttt{float32}} & \multicolumn{2}{c}{\texttt{bfloat16}} & \multicolumn{5}{c}{\texttt{float16}} \\
& \multicolumn{1}{c}{torchdiffeq} & \multicolumn{1}{c}{rampde} & \multicolumn{1}{c}{torchdiffeq} & \multicolumn{1}{c}{rampde} & \multicolumn{2}{c}{torchdiffeq} & \multicolumn{3}{c}{rampde} \\
Metric & & & & & Grad & None & Grad & None & Dyn \\
\midrule
\textbf{2-spirals} & & & & & & & & & \\
\addlinespace[0.5ex]
\quad Val Loss & 2.671 & 2.670 & 2.667 & 2.674 & 2.666 & 9.065 & 2.679 & 2.664 & 2.673 \\
\quad Avg Fwd Time (s) & 0.41 & 0.27 & 0.47 & 0.32 & 0.48 & 0.49 & 0.33 & 0.33 & 0.33 \\
\quad Avg Bwd Time (s) & 0.60 & 0.88 & 0.68 & 1.03 & 0.70 & 0.70 & 1.16 & 1.15 & 1.68 \\
\quad Total Time (s) & 2026.2 & 2312.0 & 2304.1 & 2694.3 & 2362.9 & 2377.7 & 2980.1 & 2961.7 & 4009.8 \\
\quad Max Memory & 1.3GB & 35.3MB & 919.5MB & 29.6MB & 919.5MB & 919.5MB & 29.7MB & 29.7MB & 29.5MB \\
\addlinespace
\textbf{8-Gaussians} & & & & & & & & & \\
\addlinespace[0.5ex]
\quad Val Loss & 2.805 & 2.882 & 2.888 & 2.859 & 2.831 & 9.448 & 2.840 & 2.828 & 2.879 \\
\quad Avg Fwd Time (s) & 0.42 & 0.28 & 0.48 & 0.33 & 0.49 & 0.49 & 0.32 & 0.33 & 0.33 \\
\quad Avg Bwd Time (s) & 0.61 & 0.89 & 0.69 & 1.04 & 0.72 & 0.71 & 1.15 & 1.16 & 1.71 \\
\quad Total Time (s) & 2047.5 & 2329.2 & 2341.1 & 2747.4 & 2407.5 & 2389.6 & 2949.1 & 2985.9 & 4082.5 \\
\quad Max Memory & 1.3GB & 35.3MB & 919.5MB & 29.6MB & 919.5MB & 919.5MB & 29.7MB & 29.7MB & 29.5MB \\
\addlinespace
\textbf{checkerboard} & & & & & & & & & \\
\addlinespace[0.5ex]
\quad Val Loss & 3.565 & 3.592 & 3.565 & 3.560 & 3.549 & 12.366 & 3.604 & 3.566 & 3.573 \\
\quad Avg Fwd Time (s) & 0.41 & 0.27 & 0.48 & 0.33 & 0.48 & 0.49 & 0.33 & 0.33 & 0.33 \\
\quad Avg Bwd Time (s) & 0.61 & 0.89 & 0.69 & 1.05 & 0.71 & 0.71 & 1.15 & 1.17 & 1.70 \\
\quad Total Time (s) & 2038.6 & 2328.9 & 2336.2 & 2755.6 & 2384.2 & 2399.6 & 2958.0 & 3005.6 & 4068.0 \\
\quad Max Memory & 1.3GB & 35.3MB & 919.5MB & 29.6MB & 919.5MB & 919.5MB & 29.7MB & 29.7MB & 29.5MB \\
\bottomrule
\end{tabular}%
}
\end{table}

The two ODE solver implementations differ in their computational cost due to their different memory-compute tradeoffs.
While \texttt{torchdiffeq} stores all intermediate states and computational graphs during the forward pass, \texttt{rampde} saves only the states and recomputes the ODE dynamics during backpropagation, which lowers memory usage at the cost of additional computation.
This results in a large reduction of memory usage  (1.3GB for \texttt{torchdiffeq} to 35MB for \texttt{rampde} in \texttt{float32}) while increasing the backward pass time by approximately 50\% for \texttt{rampde}.
For these small-scale 2D problems, reducing the precision did not yield wall-clock speedups. This is probably due to the small problem size, which leads to poor utilization of the hardware's tensor cores and the overhead of precision conversions.

Next, we illustrate the behavior of roundoff errors in both approaches to empirically support our theoretical analysis in~\cref{sec:roundoff}.
To this end, we use a network trained using \texttt{rampde} in \texttt{float32} using the 2 spirals dataset and evaluate the relative errors of state and gradients to double precision as the number of time steps increases.
In Figure~\ref{fig:cnf_roundoff}, we show the relative errors for \texttt{float16} with proper scaling, that is, gradient scaling for \texttt{torchdiffeq} and dynamic adjoint scaling for \texttt{rampde}.
It can be seen that both approaches accurately approximate the terminal state and the initial adjoint, which is the derivative for the initial value.
For \texttt{torchdiffeq} we notice a slight positive trend in the errors of the weight gradients in this example. As expected from our theory, the error in the weight gradients is comparable for the different step sizes in \texttt{rampde}.

\begin{figure}
\centering
\includegraphics[width=.9\textwidth]{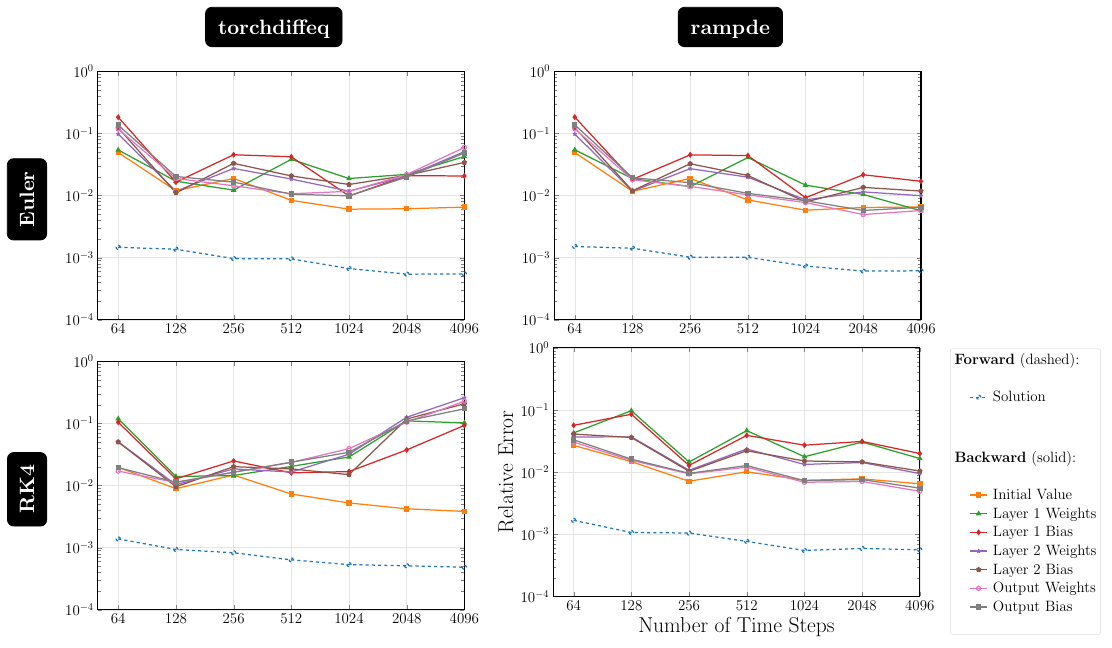}
\caption{Roundoff error experiment for the 2D Continuous Normalizing Flow
  training using the 2-spirals dataset in \cref{sub:cnf_experiment}. We compare the accuracy of forward (dashed) and backward propagation (solid) in \texttt{float16} for
    \texttt{torchdiffeq} (left) and \texttt{rampde} (right)
  for forward Euler (top) and RK4 (bottom) as the number of time steps increases.
  We measure the relative error of each quantity specified in the legend against the \texttt{float64} computation.
The accuracy of the solution and gradients for the initial value is consistent for \texttt{torchdiffeq}, while the gradients of the neural network weights become slightly less accurate as $N$ grows.
Consistent with our analysis in~\cref{sec:roundoff}, the relative errors for \texttt{rampde} remain stable. }
\label{fig:cnf_roundoff}
\end{figure}

\subsection{Optimal Transport Flows}
\label{sub:otflow_experiment}

We test the impact of mixed precision training on a CNF example with optimal transport regularization using the BSDS300 dataset~\cite{MartinEtAl2001BSDS300} following the experimental setup in our previous work~\cite{OnkenEtAl2020OTFlow}.
The dataset is a 63-dimensional feature representation derived from natural image patches in the original Berkeley segmentation dataset.
The dataset size makes it an attractive target for mixed precision training because of the large memory and computational costs and because, when adding the time to the ODE state as done in OT-Flow, the network is evaluated with 64-dimensional inputs, which increases tensor core utilization.

Unlike standard CNFs, OT Flows adds transport costs to regularize the training objective and uses the underlying optimality condition to express the optimal velocity as the negative gradient of the value function.
The training objective reads
\begin{equation}
    \mathcal{L}(\bftheta) = \mathbb{E}_{\bfx\sim\pi_{\rm data}} \left[\alpha_1 \left(\frac12\| \bfz(1) \|^2 - \ell(1)\right) +  c_{\rm L}(1) + \alpha_2 c_{\rm HJB}(1)\right]
\end{equation}
where $\bfz(1), \ell(1), c_{\rm L}(1), c_{\rm HJB}(1)$
track the trajectory of the sample, the volume change, the transport costs, and the violation of the Hamilton-Jacobi-Bellman (HJB) equation, respectively.
They are the terminal states of the augmented neural ODE with states $\bfy(t) = [\bfz(t), \ell(t), c_L(t), c_{\rm HJB}(t)]^\top$  and dynamics
\begin{equation*}
\frac{d\bfy}{dt} = f_{\rm OT}(t, \bfy(t), \bftheta) =
\begin{bmatrix}
    - \nabla V(t, \bfy(t),\bftheta) \\ \Delta V(t,\bfy(t),\bftheta) \\ \frac12 \|\nabla V(t,\bfy(t),\bftheta)\|^2 \\
    \left| \frac{\partial V}{\partial t}(t,\bfy(t),\bftheta) + \frac{1}{2}\|\nabla V(t,\bfy(t),\bftheta)\|^2\right|
\end{bmatrix}, \quad \bfy(0) =
    \begin{bmatrix}
    \bfx \\ 0 \\ 0 \\ 0
    \end{bmatrix}.
\end{equation*}
The OT-Flow approach parameterizes the value function $V: [0,1]\times \R^d \times \R^p \to \R$ and computes its gradient and Laplacian analytically.
For this experiment, we use a wider version of the architecture as in~\cite{OnkenEtAl2020OTFlow}, which consists of an opening layer with 64 inputs (concatenating $\bfz(t)$ and $t$) and 1024 outputs, followed by one residual layer and one affine layer that maps the 1024 outputs to a scalar.
The activation function is the antiderivative of the $\tanh$ function, which we compute in single precision.
As in~\cite{OnkenEtAl2020OTFlow}, we use $\alpha_1=800, \alpha_2=2000$.
Sample quality is evaluated using the Maximum Mean Discrepancy (MMD) metric in addition to the individual loss components.
The ODE is integrated using RK4 with 16 equidistant time steps during both training and 30 during evaluation.
We train for up to 10,000 iterations with batch size 512, employing the same stopping criteria as in the original OT-Flow experiments to enable early stopping when the validation loss stops improving.

Table~\ref{tab:otflowlarge_results} presents the results on the BSDS300 dataset using a set-up similar to that of the CNF experiment.
Our baseline in this experiment is TensorFloat-32 (\texttt{tfloat32}), which is the default \texttt{float32} mode on modern NVIDIA GPUs (Ampere and newer). Unlike pure \texttt{float32}, \texttt{tfloat32} internally uses mixed precision within tensor cores. Specifically, it rounds \texttt{float32} inputs to 10-bit mantissas while maintaining the 8-bit exponent range, effectively performing computations in a reduced precision before accumulating in \texttt{float32}.
This hardware-level mixed precision provides speed benefits while maintaining the programming interface of \texttt{float32}.
In this experiment, \texttt{float16} training fails to converge without gradient scaling due to NaN in the gradients in the beginning of training.
This indicates that the limited range of \texttt{float16} requires careful scaling or hyperparameter tuning to prevent gradient underflow.
With appropriate scaling strategies, all precision configurations achieve comparable validation metrics: negative log-likelihoods range from -159 to -167, transport costs remain bounded between 2.8 and 3.3, and HJB penalties stay close to 1.0, indicating that the mixed precision schemes perform comparably to single precision.
The Maximum Mean Discrepancy (MMD) values, which measure sample quality, are comparable ($\in [0.02, 0.05]$) in all successful configurations, indicating a comparable match between the data and the generated distributions across the precisions.

\begin{table}[t]
\centering
\caption{
    Mixed precision training results for optimal transport flows (OT-Flow) on the 63-dimensional BSDS300 feature dataset.
    We test torchdiffeq and rampde across \texttt{tfloat32}, \texttt{bfloat16}, and \texttt{float16} precisions and for the latter use different scaling strategies (Grad=gradient scaling, None=no scaling, Dyn=dynamic
  adjoint scaling).
  Both solvers fail to converge in
   \texttt{float16} without gradient scaling (indicated by "---").
   With proper scaling, all configurations achieve similar validation metrics but vary in computational costs.
   By avoiding re-computations in the backward pass, torchdiffeq generally is faster but requires more memory.
   In contrast, rampde achieves significant memory savings at some additional costs during the backward pass.
   Both schemes see slight speedups in 16-bit precisions compared to \texttt{tfloat32}.
}
\label{tab:otflowlarge_results}
\resizebox{\textwidth}{!}{%
\begin{tabular}{lccccccccc}
\toprule
& \multicolumn{2}{c}{\texttt{tfloat32}} & \multicolumn{2}{c}{\texttt{bfloat16}} & \multicolumn{5}{c}{\texttt{float16}} \\
& \multicolumn{1}{c}{torchdiffeq} & \multicolumn{1}{c}{rampde} & \multicolumn{1}{c}{torchdiffeq} & \multicolumn{1}{c}{rampde} & \multicolumn{2}{c}{torchdiffeq} & \multicolumn{3}{c}{rampde} \\
Metric & & & & & Grad & None & Grad & None & Dyn \\
\midrule
\quad Val Loss & -128767.7 & -124379.1 & -129675.6 & -128714.9 & -131398.6 & --- & -131309.3 & --- & -129021.0 \\
\quad Val L & 3.001 & 3.292 & 2.986 & 2.871 & 3.041 & --- & 3.059 & --- & 3.114 \\
\quad Val NLL & -163.7 & -159.0 & -164.8 & -163.4 & -167.0 & --- & -166.9 & --- & -164.0 \\
\quad Val HJB & 1.079 & 1.416 & 1.061 & 1.000 & 1.093 & --- & 1.098 & --- & 1.086 \\
\quad Val MMD & 0.0332 & 0.0237 & 0.0285 & 0.0469 & 0.0285 & --- & 0.0328 & --- & 0.0276 \\
\quad Avg Fwd Time (s) & 0.246 & 0.195 & 0.211 & 0.163 & 0.212 & --- & 0.161 & --- & 0.162 \\
\quad Avg Bwd Time (s) & 0.374 & 0.613 & 0.310 & 0.515 & 0.310 & --- & 0.537 & --- & 0.780 \\
\quad Avg Time per Step (s) & 0.620 & 0.808 & 0.521 & 0.677 & 0.521 & --- & 0.698 & --- & 0.942 \\
\quad Total Time (s) & 6016.0 & 5898.5 & 5160.7 & 6569.3 & 5005.2 & --- & 6701.3 & --- & 9140.8 \\
\quad Total Iterations & 9,700 & 7,300 & 9,900 & 9,700 & 9,600 & --- & 9,600 & --- & 9,700 \\
\quad Max Memory (GB) & 18.9 & 2.0 & 11.1 & 1.5 & 11.1 & --- & 1.5 & --- & 1.5 \\
\bottomrule
\end{tabular}%
}
\end{table}

Compared to the small-scale CNF experiment, the memory savings of the mixed precision training for each scheme are more pronounced (almost 50\% savings for \texttt{torchdiffeq} and 25\% savings for \texttt{rampde}).
As before, \texttt{rampde} significantly reduces the memory footprint in each configuration.
We observe a modest speedup in the \texttt{float16} and \texttt{bfloat16} configurations compared to \texttt{tfloat32}, especially in the average time of the forward passes.
As expected, \texttt{rampde}'s backward pass takes longer than in \texttt{torchdiffeq} due to the recomputation of the dynamics.
The overall training time varies slightly due to the early stopping.

This experiment demonstrates that mixed precision training can be attractive for high-dimensional optimal transport problems, with the choice of precision and solver depending on the specific computational constraints.
For memory-limited scenarios typical in large-scale generative modeling, \texttt{rampde} with \texttt{bfloat16} offers an excellent balance, providing $10 \times$ memory reduction with only modest speed penalties.
When training speed is paramount and memory is available, \texttt{torchdiffeq} with \texttt{bfloat16} delivers the fastest convergence.

\subsection{STL-10 Image Classification}
\label{sub:stl10_experiment}

To evaluate the potential of our mixed precision scheme to reduce memory consumption and improve training time on a large-scale learning problem, we apply it to train a  Neural ODE-based convolutional neural network to classify the STL-10 dataset.
Rather than improving the state-of-the-art accuracy, we focus on the computational aspects of this problem and the comparison of our algorithm to \texttt{torchdiffeq} in different precisions.

The STL-10 dataset contains 13,000 labeled natural images divided into 10 classes and is derived from the larger ImageNet dataset. We split the 5,000 training images into 4,000 training and 1,000 validation images, and evaluate on the full set of 8,000 test images. Each image has three color channels for RGB and is $96\times 96$ pixels.

In order to generalize effectively from only 500 labeled training images per class, we use data augmentation during training.
We upscale the images to $128\times128$ to use the tensor cores efficiently, and we use a random resize crop with a uniform scale sampled from $[0.5, 1.0]$ and aspect ratio from $[0.75, 1.33]$.
We then apply a random horizontal flip (probability 0.5) and apply a color jitter to manipulate the brightness, contrast, saturation, and hue with respective maximum adjustments of 0.4, 0.4, 0.4, and 0.1.
To encourage the model to focus more on shapes, texture, and edges rather than relying on color information, we convert the image to grayscale with 0.1 probability.
Finally, we normalize the image to zero mean and unit variance using the per-channel mean and standard deviation computed from the training images.
For the validation and test images, we only perform center cropping to $128\times 128$, tensor conversion, and normalization.

We use a version of the parabolic CNN architecture proposed in~\cite{RuthottoHaber2020} with about 12.6M trainable weights.
This architecture provides a particularly challenging test case for mixed precision training because its hyperparameter choice is less well understood and it is critical to avoid precision-specific hyperparameter optimization.
The architecture consists of several blocks, three of which are defined as neural ODEs.
The first block is an opening layer that maps the RGB images to 128 channels with a $3\times3$ convolution followed by instance normalization and ReLU activation.
The output of the opening layer is used as the initial state of the first neural ODE block governed by the dynamics
\begin{equation}
    \label{eq:NODE_STL10}
    \frac{d}{dt} \bfy(t) = - \bfW(t)^\top \mathcal{N}\left(\sigma\left( \bfW(t) \bfy(t) + \bfb(t)\right)\right), \quad t \in (0,1].
\end{equation}
Here, $\bfb$ is a bias,  $\bfW$ is a convolution operator with $128$ input and output channels and $3\times 3$ stencils, $\mathcal{N}$ denotes the instance normalization, and $\sigma$ is the ReLU activation.
We discretize the neural ODE with four equidistant RK4 steps and model the weights and biases as piecewise constants on each time interval.
The terminal state of the neural ODE block is then fed into a connecting block that consists of a $3\times 3$ convolution that doubles the number of channels, instance normalization, ReLU activation, and an average pooling layer to halve the number of pixels in each direction.
This is followed by another neural ODE block of the form~\eqref{eq:NODE_STL10} but with twice the number of channels, followed by another connector, and a final neural ODE block with 512 channels and images of size $32\times 32$.
Finally, the feature channels are global averaged to reduce sensitivity to translations before the final fully connected layer. The loss function is a softmax cross-entropy loss.

We investigate whether evaluating the forward pass and backward pass in 16-bit precisions using PyTorch's autocast context can lead to similar performance to single precision training at reduced computational cost.
For \texttt{rampde}, this means that the three neural ODE blocks are evaluated with our custom mixed precision handling, while the opening and connecting blocks use the default whitelisted and blacklisted operators.
We use the same strategy to evaluate the model when testing our rampde package and the baseline, torchdiffeq.

We compare \texttt{tfloat32}, \texttt{bfloat16}, and \texttt{float16} floating-point precisions and for the latter we experiment with gradient scaling, dynamic scaling (rampde only) and no scaling. Each configuration is trained with SGD (momentum 0.9, weight decay $5 \times 10^{-4}$) for 160 epochs, batch size 16, initial learning rate 0.05, and a cosine annealing schedule. All experiments use a fixed random seed of 25 for data splitting and model initialization to ensure reproducibility. We calculate the loss and accuracy on the training and validation sets and plot the training loss versus wall-clock time in~\cref{fig:stl10_train_loss}.
After the training is completed, we compute the accuracy and loss on the canonical 8,000-image STL-10 test split.
We summarize all key performance metrics and additional statistics (training time, peak memory) in~\cref{tab:stl10_results}.

\begin{figure}[t]
    \centering
\includegraphics[width=.63\textwidth, trim=0 0 0 15, clip=true]{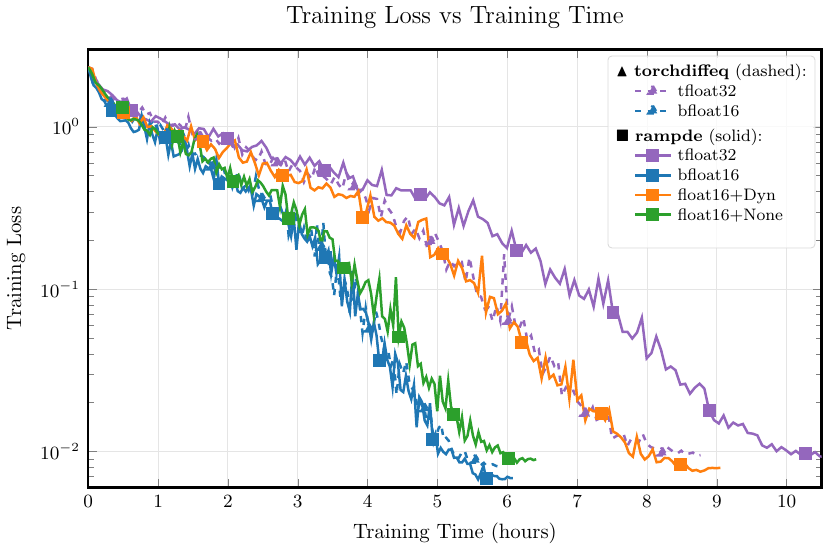}
\caption{Training loss convergence for the STL-10 experiment in \cref{sub:stl10_experiment} comparing torchdiffeq (triangular markers, dashed lines) and rampde (square markers, solid lines) in  different precisions. The plot shows training loss (log scale) versus wall-clock time in hours, demonstrating that mixed precision variants can achieve comparable final loss values more quickly. For clarity, we omit rampde in \texttt{float16} with gradient scaling as its curve is similar to that without scaling. We also omit \texttt{torchdiffeq} at \texttt{float16} as it failed with NaN after a few training steps even with gradient scaling.}
\label{fig:stl10_train_loss}
\end{figure}

\begin{table}[t]
\centering
\caption{
STL10 image classification performance for torchdiffeq and rampde.
 Our rampde implementation achieves similar accuracy in 16-bit than in single precision while reducing the memory footprint and training time by almost a factor of two. With \texttt{float16} precision it works well with gradient scaling (Grad) but achieves slightly better test accuracy with dynamic scaling (Dyn) albeit at some additional cost during backward.  Torchdiffeq fails in \texttt{float16} without loss scaling due to a NaN.
}
\label{tab:stl10_results}

\resizebox{\textwidth}{!}{%
\begin{tabular}{lcccccccccc}
\toprule
& \multicolumn{2}{c}{\texttt{float32}} & \multicolumn{2}{c}{\texttt{tfloat32}} & \multicolumn{2}{c}{\texttt{bfloat16}} & \multicolumn{4}{c}{\texttt{float16}} \\
& \multicolumn{1}{c}{torchdiffeq} & \multicolumn{1}{c}{rampde} & \multicolumn{1}{c}{torchdiffeq} & \multicolumn{1}{c}{rampde} & \multicolumn{1}{c}{torchdiffeq} & \multicolumn{1}{c}{rampde} & torchdiffeq & \multicolumn{3}{c}{rampde} \\
Metric & & & & & & & Grad & Dyn & Grad & None \\
\midrule
\textbf{STL10} &  &  &  &  &  &  &  &  &  &  \\
\addlinespace[0.5ex]
\quad Val Acc & 0.762 & 0.797 & 0.742 & 0.780 & 0.735 & 0.805 & 0.110 & 0.783 & 0.784 & 0.775 \\
\quad Val Loss & 0.878 & 0.752 & 0.934 & 0.765 & 0.914 & 0.667 & 2.326 & 0.790 & 0.759 & 0.824 \\
\quad Test Acc & 0.742 & 0.786 & 0.746 & 0.792 & 0.747 & 0.795 & 0.100 & 0.778 & 0.787 & 0.771 \\
\quad Test Loss & 0.929 & 0.768 & 0.885 & 0.745 & 0.890 & 0.732 & 2.337 & 0.798 & 0.748 & 0.808 \\
\quad Avg Fwd Time (s) & 0.29 & 0.28 & 0.28 & 0.27 & 0.19 & 0.15 & 1.16 & 0.15 & 0.15 & 0.15 \\
\quad Avg Bwd Time (s) & 0.55 & 0.77 & 0.51 & 0.73 & 0.34 & 0.40 & 1.09 & 0.67 & 0.43 & 0.43 \\
\quad Max Memory & 21.5GB & 4.5GB & 21.5GB & 4.3GB & 6.9GB & 2.2GB & 6.9GB & 2.4GB & 2.3GB & 2.3GB \\
\bottomrule
\end{tabular}%
}
\end{table}

Although we used the same random seed for the data split and initialization, we observed that training is not fully deterministic, which is likely due to non-deterministic low-level hardware behavior.

At \texttt{float16}, torchdiffeq fails early on in the training as the loss is beyond the representable range (at iteration 9 without gradient scaling and iteration 31 with scaling).
This demonstrates the need for careful combination of the precisions or more flexible dynamic scaling schemes when training neural ODEs in \texttt{float16}.
In contrast, our rampde implementation successfully trains in \texttt{float16} even without scaling, though dynamic scaling provides the best results.

This experiment indicates that MPT can achieve memory and runtime reduction for large-scale problems while maintaining competitive accuracy.
Comparing the memory usage across implementations, rampde reduces peak memory from 21.5GB to 4.3GB in single precision, approximately a 5$\times$ reduction.
For runtime performance, the fastest configuration overall is torchdiffeq with \texttt{bfloat16} (0.19s forward, 0.34s backward per iteration).
The fastest rampde configuration (\texttt{bfloat16}: 0.15s forward, 0.40s backward) achieves comparable forward pass speed and only 43\% overhead in the backward pass, while using 3× less memory (2.2GB vs 6.9GB).
Within rampde, mixed precision variants provide approximately 1.8× speedup compared to \texttt{tfloat32} (0.27s vs 0.15s forward pass).
Most importantly, all configurations maintain test accuracy above 77\% when proper scaling is used, validating that neural ODEs can be effectively trained in mixed precision without sacrificing model performance or requiring precision-specific hyperparameter tuning.
A step-size stability experiment on one of the trained STL-10 neural ODE blocks, sweeping $h\rho(\bfJ)/2$ across the Dahlquist boundary for the precision formats considered here, is reported in Appendix~\ref{sec:sm_stepsize}.

\section{Discussion}\label{sec:discussion}

We developed and analyzed mixed precision training (MPT) techniques for Neural ODEs that take advantage of low-precision support in modern hardware and improve scalability to larger problem sizes. Neural ODEs present unique challenges for existing MPT frameworks due to the potential accumulation of roundoff errors over many time steps. Following the intuition from roundoff error analysis and existing MPT guidelines, we accumulate in high precision while computing expensive neural network evaluations in low precision. We carefully analyze roundoff errors, provide an efficient implementation, and thoroughly validate our approach. Our results demonstrate that mixed precision neural ODEs can achieve accuracy comparable to single-precision training while offering significant memory reduction and, for adequately sized problems, also reduced training times.

Our MPT approach has three main pillars. First, we designed explicit ODE solvers that maintain states and accumulate in high precision (\texttt{float32}) while evaluating the neural network in low precision (\texttt{float16} or \texttt{bfloat16}) and implemented a custom backward propagation scheme that uses a similar principle for differentiation. Second, we developed a dynamic adjoint scaling scheme that adapts layer-wise scaling ideas to the continuous-time setting and automatically maintains the range during backpropagation. Third, our roundoff error analysis proves that this intuitive approach is theoretically sound in that the relative errors remain $\mathcal{O}(\ulow)$ and do not accumulate uncontrollably with the number of time steps for reasonable step sizes. The CNF experiments empirically confirm that roundoff errors of our scheme behave as predicted by our theoretical analysis while that of our benchmark method, torchdiffeq, shows slight increase of errors for smaller step sizes.

Our experiments demonstrate the potential of mixed precision training for neural ODEs. In the STL-10 image classification task, we achieved almost a $2\times$ speedup and $2\times$ memory reduction using \texttt{float16} and \texttt{bfloat16} precisions compared to \texttt{tfloat32}, while maintaining over 76\% test accuracy across all properly scaled configurations. The OT-Flow experiments on 63-dimensional data showed modest speedups with \texttt{float16} and \texttt{bfloat16} compared to \texttt{tfloat32}, validating the approach for moderately high-dimensional generative modeling tasks.

Based on our theoretical and empirical findings, we offer the following recommendations for using mixed precision neural ODEs. We emphasize that MPT is most beneficial for problems with moderate accuracy requirements where data is inherently noisy and that are sufficiently large such that memory and the computational time of tensor-core friendly operations are limiting factors. In such situations, \texttt{rampde} with \texttt{bfloat16} provides a good accuracy-memory tradeoff and may be sufficient. If more accuracy is required, \texttt{rampde} with gradient or dynamic scaling is worth considering and in our experiments did not require additional hyperparameter tuning that is otherwise common in \texttt{float16}. Finally, MPT can be attractive even if training times are higher when deployment requires low-precision inference, as MPT directly optimizes quantized neural ODEs and can avoid post-training quantization errors.

Our approach has several limitations that suggest directions for future research. Performance gains are hardware-dependent, with best results on GPUs with tensor core support. Our current implementation focuses on explicit solvers, which are standard in deep learning applications where the nonlinearity of neural networks makes implicit methods less attractive.
Several research directions appear promising. Extending to 8-bit quantization would provide extreme compression, though our theoretical analysis suggests this may be challenging with $\mathcal{O}(\ulow) \approx 2^{-3}$ or  $\mathcal{O}(\ulow) \approx 2^{-4}$ depending on the format. Adapting the framework to neural stochastic differential equations (neural SDEs) would broaden applicability to score-based generative models and stochastic control problems. Applying these techniques to optimal control problems and PDE-constrained optimization, where neural ODEs show increasing promise, could enable new applications.

\section*{Acknowledgements}

We thank Deepanshu Verma for his help in investigating different mixed precision training frameworks and many useful discussions.
In accordance with SIAM's editorial policy on AI usage, we describe our use of AI tools in as much detail as possible: We used generative search and ChatGPT's DeepResearch to assist in discovering related literature in the machine learning domain. We used Claude Code with Sonnet 4 and Opus 4 as a coding assistant to: (1) refactor and optimize our prototype implementations to a publishable state enabling reproducibility; (2) expand unit test coverage; (3) help document the code; (4) develop reproducible scripts for generating all figures and tables presented in this paper. These scripts are available in our repository. Additionally, we used LLMs to polish the manuscript text for grammar, clarity, and consistency of mathematical notation.
We take full responsibility for the accuracy and integrity of all content in this paper.

\appendix

\section{Numerical Exploration of Step Size Stability}
\label{sec:sm_stepsize}

We perform a numerical experiment to examine whether
reducing the working precision of the ODE right-hand side narrows the
admissible step-size range of an explicit integrator applied to one of
the trained neural ODE blocks used in the STL-10 classifier of
\cref{sub:stl10_experiment}. The experiment complements
\cref{theo:main_result}, whose bound
$\mathcal{O}(\ulow) + \mathcal{O}(\uhigh/h)$ separates the two
contributions to the accumulated error but does not, on its own, address
the classical Dahlquist stability picture for a given integrator. Our setting is similar to the experiments in~\cite{croci2022mixed} and our observations are consistent
with the order-preservation result of their Theorem~3.6.

\subsection{Test problem}

We use the middle of the three neural ODE blocks in the trained STL-10
network (denoted \texttt{ode2} in our implementation). The block has
$C = 256$ channels and spatial size $64\times 64$, and its weights
$\bfW(t) \in \R^{C \times C \times 3 \times 3}$ and biases
$\bfb(t) \in \R^{C}$ are learned as piecewise-constant functions of $t$
on four intervals; see~\cref{eq:NODE_STL10}. We use the weights from the
first interval so that the right-hand side is autonomous. Writing $\bfK = \bfW(0)$ and
$\mathbf{b} = \bfb(0)$, we compare two right-hand sides:
\begin{align*}
  \text{linearized:} &\quad \frac{d}{dt} \bfy = -\bfK^\top (\bfM \odot \bfK \bfy), \\
  \text{parabolic:}  &\quad \frac{d}{dt} \bfy = -\bfK^\top \sigma(\bfK \bfy + \mathbf{b}),
\end{align*}
with $\sigma = \mathrm{ReLU}$ and the mask
$\bfM = \mathbf{1}\{\bfK \bfy^0 + \mathbf{b} > 0\}$ frozen at the
initial state $\bfy^0$. The Jacobian of the linearized block,
$\bfJ = -\bfK^\top \diag(\bfM)\bfK$, is constant and symmetric negative
semidefinite. The parabolic block has the same symmetric negative
semidefinite structure~\cite[Assumption~3.2]{croci2022mixed} but carries
a state-dependent mask and is therefore nonlinear.

Initial conditions are taken from the STL-10 test set: four images,
resized to $128 \times 128$ and normalized with the training
statistics, are pushed through the trained opening and first-ODE layers
of the network so that $\bfy^0$ follows the activation distribution
seen at inference. The spectral radius $\rho(\bfJ)$ is estimated by
$50$ power iterations on double-precision Jacobian-vector products,
and the maximum over the four samples is used to set the step size.

\subsection{Integrator and sweep}

We integrate with forward Euler using step size
$h = \alpha \cdot 2/\rho(\bfJ)$ for
$\alpha \in \{0.5, 0.99, 1, 2, 3, 4, 5, 6, 7\}$, covering the interior
of the stability region, its boundary, and the unstable
regime. The number of steps is fixed at $N = 1000$ for both the
forward and backward propagation, so the horizon
$T = N h$ grows linearly with $\alpha$. This choice exposes accumulated
roundoff. We compare six precision modes: \texttt{float64} and
\texttt{float32} reference paths; \texttt{rampde} with the default
\texttt{float32} accumulator and a \texttt{bfloat16} or
\texttt{float16} right-hand side (the latter with dynamic scaling); and
two na{\"i}ve mixed-precision baselines that store the state in
\texttt{bfloat16} or \texttt{float16} without a high-precision
accumulator. We report three diagnostics:
(i)~$\|\bfy^N\|_2 / \|\bfy^0\|_2$, the forward growth;
(ii)~$\angle(\bfg_{\text{mode}}, \bfg_{\text{fp64}})$, the angle in
radians between the adjoint gradient in each mode and the double
precision reference, with
$\bfg = \nabla_{\bfy^0} \tfrac{1}{2}\|\bfy^N\|_2^2$;
(iii)~the relative magnitude error
$\big|\|\bfg_{\text{mode}}\| - \|\bfg_{\text{fp64}}\|\big| /
\|\bfg_{\text{fp64}}\|$.

\begin{figure}[t]
\centering
\includegraphics[width=\textwidth]{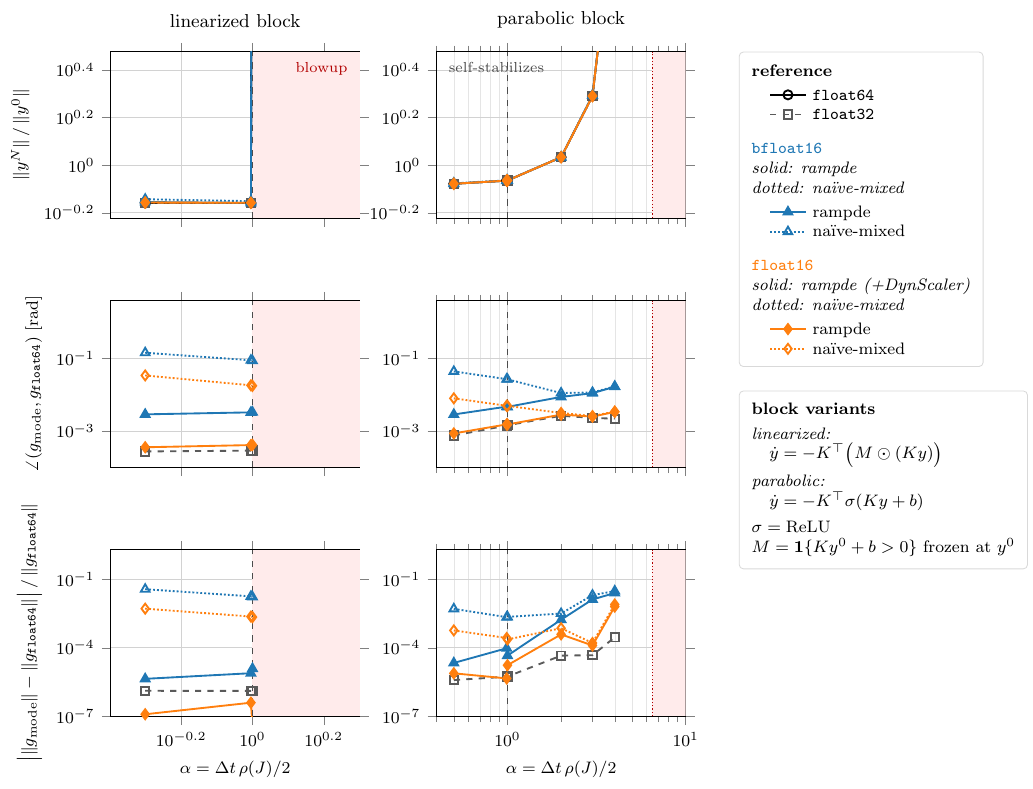}
\caption{Step-size stability diagnostics for the linearized (left) and
parabolic (right) blocks across six precision modes. Top row: forward
growth $\|\bfy^N\|_2/\|\bfy^0\|_2$ versus
$\alpha = h\rho(\bfJ)/2$. Middle row: adjoint direction error in
radians. Bottom row: relative adjoint magnitude error. Shaded regions
mark the empirical blow-up regime; the dashed vertical line in the
left column marks the Dahlquist boundary $\alpha = 1$. Solid lines
denote \texttt{rampde} with a \texttt{float32} accumulator; dotted
lines denote na{\"i}ve mixed-precision without the accumulator.}
\label{fig:sm_stability}
\end{figure}

\subsection{Observations}

\paragraph{Linearized block.}
The empirical forward blow-up boundary coincides with the Dahlquist
prediction $\alpha = 1$ for every precision considered. Inside the
stable range, $\alpha \le 0.99$, all six precisions saturate at the
structural floor $\|\bfP_{\ker(\bfJ)} \bfy^0\| / \|\bfy^0\| \approx
0.69$ set by the kernel of $\bfJ$; the differences between modes at
$\alpha = 0.5$ are below $4\%$ of this floor. The adjoint direction
errors at $\alpha = 0.5$ are summarized in~\cref{tab:sm_adjoint_angle}.
The \texttt{rampde} scheme, which evaluates the right-hand side in low
precision but accumulates the state in \texttt{float32}, produces
direction errors one to two orders of magnitude below the corresponding
na{\"i}ve mixed-precision baseline at the same mantissa width.

\begin{table}[t]
\centering
\caption{Adjoint direction error
$\angle(\bfg_{\text{mode}}, \bfg_{\text{fp64}})$ in radians at
$\alpha = 0.5$ on the linearized block of
\cref{fig:sm_stability}, read from
\texttt{gradient\_angle\_linearized.csv}.}
\label{tab:sm_adjoint_angle}
\begin{tabular}{@{}lc@{}}
\toprule
mode & $\angle(\bfg_{\text{mode}}, \bfg_{\text{fp64}})$ [rad]\\
\midrule
\texttt{float32}                       & $2.7 \times 10^{-4}$ \\
\texttt{rampde} (\texttt{bfloat16})    & $2.9 \times 10^{-3}$ \\
\texttt{rampde} (\texttt{float16})     & $3.6 \times 10^{-4}$ \\
na{\"i}ve mixed (\texttt{bfloat16})        & $1.5 \times 10^{-1}$ \\
na{\"i}ve mixed (\texttt{float16})         & $3.4 \times 10^{-2}$ \\
\bottomrule
\end{tabular}
\end{table}

\paragraph{Parabolic block.}
The state-dependent mask stabilizes the dynamics in this example: forward
trajectories remain finite up to $\alpha = 4$ and blow up at
$\alpha = 5$ for every precision, consistent with a fixed point of
$-\bfK^\top \sigma(\bfK \bfy + \mathbf{b}) = 0$. Inside the stable
range, the forward ratio saturates near $0.84$, with all six
precisions agreeing with the reference to $\mathcal{O}(10^{-3})$ at
$\alpha = 0.5$. The adjoint direction errors at $\alpha = 0.5$ are
$\mathcal{O}(10^{-3})$~rad for \texttt{rampde} (\texttt{bfloat16}),
$\mathcal{O}(10^{-4})$~rad for \texttt{rampde} (\texttt{float16}),
and $\mathcal{O}(10^{-2})$~rad for the na{\"i}ve mixed-precision
baselines; the accumulator advantage shrinks as $\alpha$ increases
and dissipation dominates, so that at $\alpha = 3$ the \texttt{bfloat16}
modes agree to $\mathcal{O}(10^{-3})$~rad. The relative adjoint
magnitude errors follow the same ordering and remain below
$5 \times 10^{-3}$ in all stable precisions.

\subsection{Interpretation}

The admissible step-size range is given by the Dahlquist boundary in
the linearized case and by the self-stabilization of the nonlinear mask
in the parabolic case, and we do not observe any narrowing of that
range as the working precision of the right-hand side is reduced to
\texttt{bfloat16} or \texttt{float16}. The observed ordering of adjoint
direction errors -- \texttt{float32}-accumulator schemes close to \texttt{float64}, na{\"i}ve
mixed-precision one to two orders looser -- is consistent with the
order-preservation result in~\cite[Theorem~3.6]{croci2022mixed} in the symmetric negative
semidefinite regime where it applies; we do not claim to prove such a
result here. In the present experiment, this regime corresponds to the
linearized block and, by virtue of the frozen mask,
approximately to the parabolic block away from the self-stabilization
boundary.

\bibliographystyle{siamplain}
\bibliography{main}

\end{document}